\documentclass[12pt]{article}
\usepackage{PRIMEarxiv}
\makeatletter
\renewcommand{\normalsize}{\@setfontsize\normalsize{12}{14}\abovedisplayskip 8pt plus 2pt minus 4pt
  \abovedisplayshortskip 0pt plus 3pt
  \belowdisplayskip \abovedisplayskip
  \belowdisplayshortskip 4pt plus 3pt minus 2pt
}
\renewcommand{\small}{\@setfontsize\small{11}{13}\abovedisplayskip 7pt plus 2pt minus 4pt
  \abovedisplayshortskip 0pt plus 3pt
  \belowdisplayskip \abovedisplayskip
  \belowdisplayshortskip 3pt plus 2pt minus 2pt
}
\makeatother
\normalsize
\usepackage[utf8]{inputenc}
\usepackage[T1]{fontenc}
\usepackage{courier}
\usepackage{amsmath,amssymb,amsthm}
\DeclareMathSizes{11}{11}{8}{6}
\usepackage{enumitem}
\usepackage{microtype}
\usepackage{graphicx}
\usepackage{booktabs}
\usepackage{placeins}
\usepackage{wrapfig}
\usepackage{algorithm}
\usepackage{algpseudocode}
\usepackage[authoryear,round]{natbib}
\usepackage{hyperref}
\usepackage{url}
\usepackage{cleveref}
\usepackage{xspace}
\usepackage{needspace}

\AddToHook{cmd/section/before}{\Needspace{6\baselineskip}}

\hypersetup{hidelinks,pdftitle={Traversing the Satisfaction-Diversity Frontier in Text-to-Image Diffusion},pdfauthor={Kevin Zhai, Siva Rajesh Kasa, Soumya Roy, Sumit Negi, Mubarak Shah}}
\newtheorem{theorem}{Theorem}
\newtheorem{lemma}[theorem]{Lemma}
\newtheorem{proposition}[theorem]{Proposition}
\newtheorem{corollary}[theorem]{Corollary}

\theoremstyle{remark}

\newcommand{\method}{SatisDive\xspace}
\newcommand{\copyrule}{late-stage latent replacement\xspace}

\title{Traversing the Satisfaction-Diversity Frontier in Text-to-Image Diffusion}
\author{Kevin Zhai\textsuperscript{1}, Siva Rajesh Kasa\textsuperscript{2}, Soumya Roy\textsuperscript{2}, Sumit Negi\textsuperscript{2}, Mubarak Shah\textsuperscript{1}\\[0.5em]
\normalfont\textsuperscript{1}University of Central Florida\qquad\textsuperscript{2}Amazon}

\begin{document}
\maketitle
\addtocontents{toc}{\protect\setcounter{tocdepth}{-1}}
\begin{abstract}
Text-to-image generation enables users to explore several images generated from the same prompt. For these generated images to be useful, each one must reflect the user's preferences, measured by a learned reward, and differ visually from the others to maintain diversity. Existing methods are limited: they either address reward and diversity separately or combine them in one aggregate score, enabling high diversity to offset low rewards. In this paper, we address these limitations by formulating generation as \emph{satisficing}: every image (candidate) must satisfy a reward floor and the batch of images must satisfy a diversity cutoff. The reward floor controls the balance between worst-candidate reward and batch diversity; we show that varying this floor defines a Pareto frontier. To traverse this frontier, we introduce \method, a training-free inference-time method. \method uses a batch-relative reward cutoff to distinguish lower- from higher-reward candidates, emphasizing reward improvement for candidates below the cutoff and diversity among candidates above it. On Pick-a-Pic, at matched DreamSim, \method improves worst-candidate reward over FK steering by up to $0.43$ with FLUX.1-dev as the base model and HPSv3 as the reward, and by up to $0.70$ with SANA-1.6B as the base model and ImageReward as the reward. More broadly, across their overlapping DreamSim ranges, \method's satisfaction--diversity curve Pareto-dominates FK steering's curve in each setting.
\end{abstract}

\begin{center}
\small\textbf{Code:} \url{https://github.com/UCF-CRCV/SatisDive}
\end{center}
\section{Introduction}
\label{sec:intro}

An important goal of text-to-image generation is to give users a choice among several high-quality candidate images. For example, a user creating a poster illustration may compare images generated from the same prompt, choosing the lighting, background, or composition that best fits the poster. Recent inference-time methods increase diversity~\citep{singh2024negtome} to provide a wider range of options, maximize a learned reward~\citep{singhal2025fk} to align generated images with user preferences, or select images by maximizing a linear combination of total image reward and batch diversity~\citep{parmar2026scaling}.

However, these approaches can still limit the user's viable choices. Firstly, reward-steering methods can concentrate sampling on higher-reward images, making the returned images visually similar~\citep{singhal2025fk,kim2025das,shekhar2026vasr}. Secondly, reward-free diversity methods make generated images more visually different, but do not explicitly account for learned reward~\citep{singh2024negtome,kim2026diverse,corso2023particle,vinograd2026eddy}. Thirdly, batch-level reward--diversity objectives~\citep{parmar2026scaling} allow greater diversity or high rewards for other images to compensate for one image's low reward.

These limitations motivate controlling worst-image reward through a reward floor, i.e., a minimum reward level required of each candidate. As the reward floor increases, the batch can become less diverse since fewer images satisfy the floor. We show that varying the reward floor defines a Pareto frontier between worst-image reward and batch diversity, which we call the satisfaction--diversity frontier. Along this frontier, we formulate image generation as \emph{satisficing}~\citep{simon1956}: \emph{every generated image must satisfy a reward floor, and the batch must satisfy a diversity cutoff.} Once a candidate clears the reward floor, further increases in reward do not improve satisfaction. To this end, we introduce \method, a training-free inference-time method. As Figure~\ref{fig:teaser} shows, \method generates visually distinct images above the reward floor, while FK steering~\citep{singhal2025fk}, a strong reward-steering baseline, returns nearly identical images.

\begin{figure}[ht]
\centering
\includegraphics[width=0.9\textwidth]{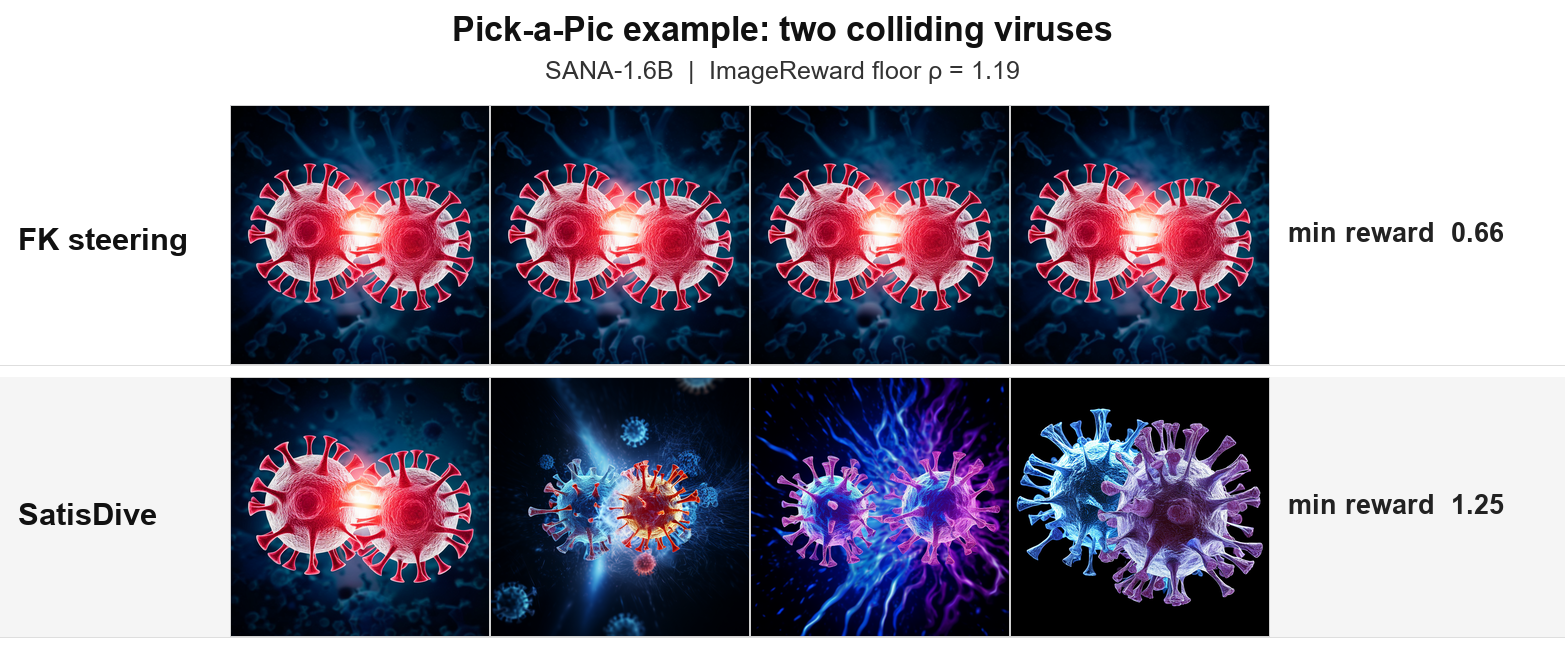}
\caption{\textbf{\method returns four visually distinct images above the reward floor, while FK steering returns four nearly identical images.} We use the SANA-1.6B~\citep{xie2025sana} flow matching model with ImageReward~\citep{xu2023imagereward} reward model. On a Pick-a-Pic~\citep{kirstain2023pickapic} prompt depicting two colliding viruses, the minimum reward is $0.66$ for FK steering and $1.25$ for \method, with reward floor $\rho{=}1.19$.}
\label{fig:teaser}
\end{figure}

Our objective traverses the satisfaction--diversity frontier through a per-image reward penalty and a batch-level diversity penalty. Under the satisficing formulation, a visually different but low-reward image should contribute less to batch diversity. Thus, these penalties are complementary: images with rewards below a batch-relative reward cutoff incur larger reward penalties and contribute less to the diversity penalty, and images with rewards above the cutoff incur smaller reward penalties but contribute more to the diversity penalty. Because some candidates may remain low reward throughout denoising, we introduce \copyrule, which replaces their latents with those of higher-reward candidates. Subsequent stochastic denoising and the diversity penalty allow the resulting trajectories to produce distinct images.

Empirically, \method improves worst-candidate reward without sacrificing batch diversity: across their overlapping DreamSim ranges, \method's worst-candidate reward is consistently higher than FK steering's, thus our satisfaction--diversity curves Pareto-dominate FK steering's on both FLUX and SANA. Our contributions are as follows.
\begin{itemize}[leftmargin=*,itemsep=4pt,parsep=0pt,partopsep=0pt,topsep=0pt]
    \item We show that different reward floors define a Pareto frontier between worst-image reward and batch diversity. Under our \emph{satisficing} formulation, every image must satisfy the selected reward floor, and the batch must satisfy a diversity cutoff.
    \item We introduce \method, a training-free inference-time method whose objective provably traverses this satisfaction--diversity frontier. The objective uses a batch-relative reward cutoff to emphasize reward improvement for lower-reward candidates and diversity among higher-reward candidates.
    \item At matched DreamSim, \method improves worst-image reward by up to $0.43$ on FLUX and $0.70$ on SANA. Our satisfaction--diversity curve Pareto-dominates the curves traced by FK steering on the FLUX.1-dev and SANA-1.6B base models.
\end{itemize}

\section{Problem Formulation}
\label{sec:problem}

A rectified-flow text-to-image model~\citep{liu2023flow} generates an image from a prompt $c$ using a learned velocity field $v_\theta(x_t,t,c)$. Generation starts from Gaussian noise $x_1\sim\mathcal N(0,I)$ and integrates $v_\theta$ from $t=1$ to $t=0$ over $T$ discretized steps, producing the image $x_0$. This formulation also covers other diffusion models such as DDPM~\citep{ho2020ddpm}, which differ in their noise schedule~\citep{lipman2023flowmatching}. Table~\ref{tab:notation} in Appendix~\ref{app:notation} summarizes the notation used throughout the paper.

Let $p_\theta(x_0\mid c)$ denote the base model's distribution. We consider methods that maintain $K$ candidates throughout a single sampling pass and generate a batch of $K$ images (denoted $x_0^{1:K}$). Given a reward $r(x_0,c)$, we first consider reward-steering methods (e.g., FK steering~\citep{singhal2025fk}) that yield high-reward images by targeting the tilted distribution $p_\lambda$:
\begin{equation}
\label{eq:tilt}
  p_\lambda(x_0\mid c)\propto p_\theta(x_0\mid c)\,\exp\!\big(\lambda\,r(x_0,c)\big),\qquad \lambda\ge 0,
\end{equation}
where larger $\lambda$ increases the relative probability of higher-reward images. These methods approximate Eq.~\ref{eq:tilt} by resampling candidate latents at intermediate timesteps with selection probabilities determined by candidate rewards; these rewards are evaluated on the model's estimate $\hat x_0$ of the final image, computed as:
\begin{equation} \label{eq:xhat}
  \hat x_0=x_t-\sigma_t v_\theta(x_t,t,c),
\end{equation}
where $\sigma_t$ is the noise level at timestep $t$. Since rectified flow is deterministic, these steering methods use the equivalent marginal-preserving stochastic form~\citep{mark2025feynman}.

\paragraph{Limitations of current methods.} The reward-only tilt in Eq.~\ref{eq:tilt} weights each candidate's selection probability during resampling according to its reward, without considering how the $K$ images differ. Proposition~\ref{prop:tilt} (proof in Appendix~\ref{app:theory}) states what happens when this tilt is applied to the current candidates during resampling.
\begin{proposition}[\textbf{As $\lambda\to\infty$, resampling produces $K$ copies of the highest-reward candidate}]
\label{prop:tilt}
At a resampling step with a unique highest-reward candidate, increasing $\lambda$ increases its expected number of copies. In the limit as $\lambda\to\infty$, the probability that all $K$ resampled candidates are copies of the highest-reward candidate approaches one.
\end{proposition}
Consequently, a greater reward tilt can reduce the number of distinct candidates retained after resampling and thereby reduce the visual diversity among the generated images.

A second class of methods increases the \emph{diversity} of generated images. NegToMe~\citep{singh2024negtome}, for instance, increases the distance between the candidates' features. Being reward-agnostic, these methods do not constrain the sampling distribution toward higher rewards. Figure~\ref{fig:toy-problem} visualizes the limitations of reward-steering and diversity methods in a toy 2D setting (details in Appendix~\ref{ssec:app-toy}).

\begin{figure*}[ht]
\centering
\includegraphics[width=\textwidth]{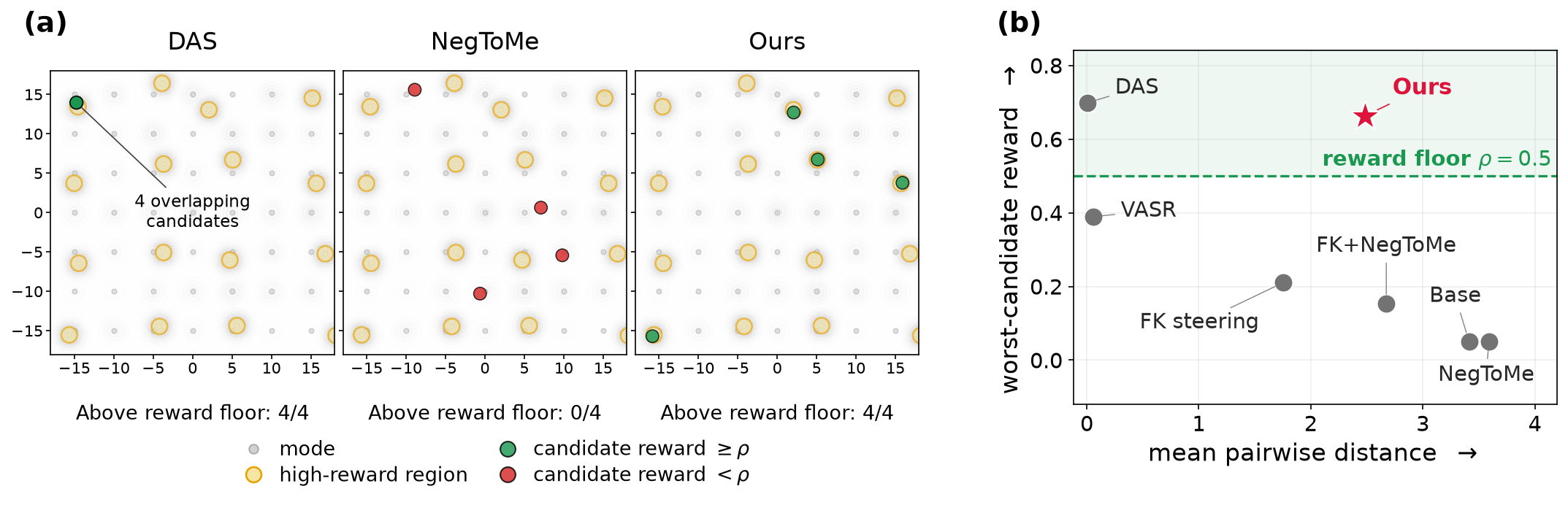}
\vspace{-6mm}
\caption{\textbf{In a toy setting, each baseline tends towards high diversity or high reward, not both.} Each method generates $K{=}4$ candidates from a frozen flow model trained on a mixture of Gaussians (details in Appendix~\ref{ssec:app-toy}). Reward is bounded in $(0,1]$, and the midpoint $\rho=0.5$ is a reward floor that every candidate must satisfy. \textbf{(a)} To visualize clearly the coverage of high-reward regions, candidates are 2D points on a Cartesian $x$--$y$ grid rather than images. We show one selected illustrative batch of $K{=}4$ candidates per method. DAS's candidates are all above $\rho$ but are not diverse; they cover only a single high-reward region (yellow circles). Conversely, NegToMe's candidates are diverse yet are all below $\rho$ (red circles). In these batches, only ours has all candidates above $\rho$ in distinct high-reward regions. \textbf{(b)} We plot worst-candidate reward against diversity (mean pairwise distance) averaged over 200 runs for each method. Methods in the shaded green region have average worst-candidate rewards greater than $\rho$. DAS achieves the highest reward but lowest diversity; NegToMe achieves the highest diversity but lowest reward.}
\label{fig:toy-problem}
\end{figure*}

\vspace{2mm}
A natural third alternative treats diversity as a second reward and samples from a joint tilt:
\begin{equation}\label{eq:joint}
  \pi_{\mathrm{joint}}(x_0^{1:K}\mid c) \propto \prod_{k=1}^K p_\theta(x_0^k\mid c)\,\exp\!\Big(\lambda\sum_{k=1}^K r(x_0^k,c) + \mu D(x_0^{1:K})\Big),\qquad \lambda,\mu\ge0,
\end{equation}
where $D$ is a batch-level diversity term and $\mu$ weights it against reward $r$. Because reward and diversity are aggregated, the joint objective can be gamed: for \(\lambda>0\), one unit of \(D\) offsets a decrease of \(\mu/\lambda\) in total reward. This exchange exists for every positive $\mu$, so $\pi_{\mathrm{joint}}$ never requires a candidate to be high-reward.

\section{Method}
\label{sec:method}

The approaches discussed in Section~\ref{sec:problem} do not jointly require a reward floor for every candidate and a minimum diversity for the whole batch. We therefore formulate generation as \emph{satisficing}~\citep{simon1956}: every candidate must satisfy a reward floor $\rho$, while the batch must satisfy a diversity requirement. These requirements define the target distribution
\begin{equation}
\label{eq:constraint}
  \pi^\star(x_0^{1:K}\mid c)\;\propto\;
  \Big(\prod_{k=1}^K p_\theta(x_0^k\mid c)\,\mathbf{1}[r(x_0^k,c)\ge\rho]\Big)
  \,\mathbf{1}[D(x_0^{1:K})\ge\delta],
\end{equation}
where $D$ measures batch diversity, $\delta$ is the required diversity level, and each indicator equals one when its condition holds and zero otherwise. Rejection sampling would draw $K$ candidates independently from $p_\theta$ and reject the entire batch if either requirement is violated. Because the probability that all $K$ candidates satisfy the reward floor decreases exponentially with $K$, rejection sampling can be prohibitively expensive.

We thus introduce \method, which replaces both indicators with differentiable penalties. Our training-free approach rests on two mechanisms: (1) a satisficing objective that increases low rewards and diversifies the batch; and (2) \copyrule, which replaces the latent of a low-reward candidate with that of a higher-reward candidate. Section~\ref{ssec:objective} details the objective and \copyrule and connects our objective to $\pi^\star$, while Section~\ref{ssec:frontier} proves that this objective traces the satisfaction--diversity Pareto frontier.

\subsection{Satisficing at Inference Time}
\label{ssec:objective}

Our objective $\mathcal J$ consists of a reward penalty and a diversity penalty, differentiable counterparts to $\pi^\star$'s two indicators. For a differentiable reward $r(\cdot)$, we minimize $\mathcal J$ over each candidate's latent $x_t^k$. At each step, \method compares each candidate's reward $r_k=r(\hat x_0^k, c)$ against a batch-relative reward cutoff $\tau$:
\begin{equation}\label{eq:tau}
  \tau = \max_k r_k - \Delta, \qquad \Delta > 0,
\end{equation}
where $\Delta$ is a fixed tolerance. Comparing against $\tau$ distinguishes low- from high-reward candidates in the batch. For candidate $k$, we use the sigmoid $q_k$ as a smooth probability that $r_k\geq\tau$:
\begin{equation}\label{eq:gate}
  q_k = \operatorname{sigmoid}\!\big((r_k - \tau)/\beta_r\big),
\end{equation}
where the temperature $\beta_r$ is computed once per run as the standard deviation of the batch's rewards; $q_k\rightarrow 1$ for rewards above $\tau$ and $q_k\rightarrow 0$ below $\tau$.

\emph{Reward Penalty.} We define a penalty for each candidate $k$ that increases with $\tau-r_k$; lower-reward candidates have larger penalties. Summing over the batch gives $\mathcal{J}_r$:
\begin{equation}\label{eq:rewpen}
  \mathcal{J}_r \triangleq -\sum_k \log q_k = \sum_k \operatorname{softplus}\!\big((\tau - r_k)/\beta_r\big) .
\end{equation}
Thus, $\mathcal{J}_r$ is the negative log-likelihood that every candidate in the batch exceeds $\tau$.

\emph{Diversity Penalty.} We measure batch diversity $D$ as the weighted mean of the pairwise cosine distances $d_{ij}$ between the pooled backbone features of candidates $i$ and $j$, each pair weighted by $q_i q_j$ (Eq.~\ref{eq:gate}). Consequently, pairs of higher-reward candidates receive greater weight than pairs with lower-reward candidates. To account for the scale of pairwise distances in each run, we set $\delta$ to the median of the $d_{ij}$, computed once per run. At each step, the diversity penalty $\mathcal{J}_d$ is then:
\begin{equation} \label{eq:divpen}
  \mathcal{J}_d \triangleq \operatorname{softplus}\!\big((\delta - D)/\beta_d\big),\quad \text{where}\quad D = \tfrac{\sum_{i<j} q_i q_j\, d_{ij}}{\sum_{i<j} q_i q_j},
\end{equation}
and $\beta_d$ is the diversity temperature. Combining the two penalties yields our objective:
\begin{equation}\label{eq:obj}
  \mathcal{J} = \mathcal{J}_r + \mathcal{J}_d = \sum_k \operatorname{softplus}\!\big((\tau - r_k)/\beta_r\big) \;+\; \operatorname{softplus}\!\big((\delta - D)/\beta_d\big).
\end{equation}
At intermediate denoising steps, we update the candidate latents using the gradients of $\mathcal J_d$ and $\mathcal J_r$ (scaling specified in Appendix~\ref{app:hyperparameters}). The objective gives greater weight to increasing reward and diversity below their cutoffs than above them. Proposition~\ref{prop:hard-limit} (proof in Appendix~\ref{app:theory}) connects the distribution induced by our objective to $\pi^\star$ (Eq.~\ref{eq:constraint}) in the zero-temperature limit. We evaluate the finite-step sampler empirically.

\begin{samepage}
\begin{proposition}[\textbf{The distribution induced by $\mathcal J$ converges to $\pi^\star$}]
\label{prop:hard-limit}
At any given timestep, let $M=\max_k r_k$ and set $\tau=\rho=M-\Delta$. Here, $D$ in $\pi^\star$ is the unweighted mean pairwise distance, obtained from the weighted mean in Eq.~\ref{eq:divpen} as all reward gates approach one. Suppose that independently sampling $K$ candidates from the base model has a nonzero chance of producing a batch with $r_k\ge\rho$ for every candidate $k$ and $D\ge\delta$ (Eq.~\ref{eq:constraint}). Suppose also that the event that $r_k=\rho$ for some $k$, and the event that $D=\delta$, each have probability zero. Then, as $\beta_r,\beta_d\to0$, the distribution induced by $\mathcal J$ converges to $\pi^\star$.
\end{proposition}
\end{samepage}

\textbf{Late-Stage Latent Replacement.} In practice, the gradient of $\mathcal{J}$ depends strongly on the structure of the latent space and the initial noise $x_1^k$, so candidates may remain in local reward optima throughout denoising. We thus propose \copyrule based on the intuition that, toward the end of denoising, it is preferable to copy from a known high-reward candidate than to continue updating a candidate that remains low reward. To see how reward may change over the remaining steps, we extrapolate each below-$\tau$ candidate's reward from the difference $r_k-r_k^{\mathrm{prev}}$ since the previous update. If this candidate's extrapolated reward is less than $\tau$, we replace its latent with that of a randomly chosen candidate above $\tau$. We apply \copyrule only in the latter half of denoising, since an early $\hat x_0^k$ estimate is a poor predictor of a candidate's final reward. After replacement, the two candidates become visually distinct by the end of denoising due to $\mathcal J_d$ and stochastic sampling. Algorithm~\ref{alg:method} details our full procedure.

\vspace{2mm}
\begin{algorithm}[!tbp]
\caption{\textbf{\method.}}
\label{alg:method}
\begin{algorithmic}[1]
\algtext*{EndIf}
\algtext*{EndFor}
\Require base model $p_\theta$, prompt $c$, reward $r$, batch size $K$, steps $T$, update steps $\mathcal G$, tolerance $\Delta$
\State sample $x^1,\dots,x^K \sim \mathcal{N}(0,I)$
\For{denoising step $s=1,\ldots,T$}
  \If{$s \in \mathcal{G}$}
    \ForAll{candidates $k$}
      \State Predict clean image $\hat x_0^k$ \Comment{\textit{Eq.~\eqref{eq:xhat}}}
      \State Compute reward $r_k = r(\hat x_0^k,c)$
    \EndFor
    \State Compute reward cutoff $\tau = \max_k r_k - \Delta$ \Comment{\textit{Eq.~\eqref{eq:tau}}}
    \State Using $\tau$, compute $\mathcal J_d$ and $\mathcal J_r$; update each $x^k$ using their gradients \Comment{\textit{Eq.~\eqref{eq:obj}}}
    \If{$s>T/2$ and $r_k^{\mathrm{prev}}$ exists for every $k$}
      \State $m=\lvert\{u\in\mathcal G:u>s\}\rvert$ \Comment{\textit{$m$ is the number of remaining update steps}}
      \State $\tilde r_k = r_k+m\cdot\max(r_k-r_k^{\mathrm{prev}},0)$, for every $k$ \Comment{\textit{Get extrapolated reward $\tilde r_k$}}
      \ForAll{$k$ such that $r_k<\tau$ and $\tilde r_k<\tau$}
        \State Sample $j\sim\mathcal U(\{i:r_i\ge\tau\})$
        \State Replace latent $x^k$ with $x^j$ \Comment{\textit{Late-stage latent replacement}}
      \EndFor
    \EndIf
    \State $r_k^{\mathrm{prev}}\gets r_k$, for every $k$ \Comment{\textit{Store rewards for the next update}}
  \EndIf
  \State Take next sampling step
\EndFor
\State \Return $x_0^{1:K}$
\end{algorithmic}
\end{algorithm}

\subsection{\method's Objective Traces the Satisfaction--Diversity Pareto Frontier}
\label{ssec:frontier}

By sweeping $\Delta$, \method's objective traces a Pareto frontier between reward and diversity. We now characterize how the reward constraint in $\pi^\star$ (Eq.~\ref{eq:constraint}) changes as the reward floor $\rho$ varies. Let $A_\rho=\{x:r(x,c)\ge\rho\}$ be the set of images that satisfy the floor, and let $p_\rho$ be the base output distribution restricted to $A_\rho$, preserving relative probabilities within this set. Because the reward constraint applies separately to each candidate, we rewrite $\pi^\star$ as a product of $K$ reward-restricted distributions and the batch diversity indicator:
\begin{equation}\label{eq:factorize}
  \pi^\star(x^{1:K}\mid c)\propto \Big(\prod_{k=1}^{K}p_\rho(x^k)\Big)\,\mathbf1[D(x^{1:K})\ge\delta],
  \qquad p_\rho(x)\propto p_\theta(x\mid c)\,\mathbf1[x\in A_\rho].
\end{equation}
We denote $\{p_\rho\}$ as the family of distributions obtained by varying the reward floor $\rho$. For each floor, Lemma~\ref{lem:iproj} (Appendix~\ref{app:theory}) shows two properties of $p_\rho$. First, among all distributions that assign probability one to $A_\rho$, $p_\rho$ uniquely minimizes the $\mathrm{KL}$ divergence to the base model. Second, the support of every such distribution with finite $\mathrm{KL}$ is contained in the support of $p_\rho$. Here, satisfaction means the highest reward floor satisfied with probability one. Theorem~\ref{thm:frontier} (proof in Appendix~\ref{app:theory}) establishes $\{p_\rho\}$ as the satisfaction--diversity Pareto frontier.

\begingroup
\setlength{\topsep}{6pt}
\begin{samepage}
\begin{theorem}[\textbf{The family $\{p_\rho\}$ forms the satisfaction--diversity Pareto frontier as $\rho$ varies}]
\label{thm:frontier}
Consider reward floors $\rho$ for which $p_\theta(A_\rho\mid c)>0$, and suppose diversity is nonincreasing under support restriction, strictly so for proper restrictions. Every candidate drawn from $p_\rho$ has reward at least $\rho$. For any two such reward floors $\rho'$ and $\rho$ with $\rho'\ge\rho$, $p_{\rho'}$ has diversity at most that of $p_\rho$. For each $\rho$, no distribution over $A_\rho$ at finite KL to $p_\theta$ is more diverse than $p_\rho$. Therefore, $\{p_\rho\}$ forms the satisfaction--diversity Pareto frontier among distributions with finite guaranteed reward and finite KL to $p_\theta$.
\end{theorem}
\end{samepage}
\endgroup

The theorem applies to diversity measures satisfying the stated support-restriction assumption; our experiments evaluate the practical satisfaction--diversity tradeoff using mean pairwise DreamSim. Reward steering instead targets the tilted distributions $p_\lambda$. For finite $\lambda$ and any reward floor $\rho$ with $0<p_\theta(A_\rho\mid c)<1$, Corollary~\ref{cor:tiltdom} (Appendix~\ref{app:theory}) shows that $p_\lambda$ assigns positive probability to images below $\rho$, unlike $p_\rho$. Thus, varying $\lambda$ traces $\{p_\lambda\}$ rather than the satisfaction--diversity Pareto frontier established by $\{p_\rho\}$.

\section{Experiments}
\label{sec:experiments}

\textbf{Models, datasets, and baselines.} We show results with three frozen base models: FLUX.1-dev~\citep{flux2024}, SANA-1.6B~\citep{xie2025sana}, and Stable Diffusion~1.5 (SD1.5)~\citep{rombach2022high}. We use two reward models, HPSv3~\citep{wu2025hpsv3} and ImageReward~\citep{xu2023imagereward}, and evaluate on the Pick-a-Pic~\citep{kirstain2023pickapic} and HPDv2~\citep{wu2023human} datasets. Our main baselines fall into two categories: 1) reward-steering methods, namely Feynman--Kac (FK) steering~\citep{singhal2025fk}, VASR~\citep{shekhar2026vasr}, and DAS~\citep{kim2025das}; and 2) diversity methods, namely Negative Token Merging (NegToMe)~\citep{singh2024negtome}, which repels candidates in feature space. We also evaluate FK+NegToMe, which combines FK steering with NegToMe's diversity mechanism. All methods use the same seeds and return $K{=}4$ images (i.e., candidates) per prompt. We report SGI~\citep{parmar2026scaling} separately in Appendix~\ref{ssec:app-sgi} because it prunes a much larger pool of $128$ candidates. Baselines use their recommended settings; sampling details are provided in Appendix~\ref{app:hyperparameters}.

\textbf{Metrics.} We evaluate images on two axes: reward and diversity. On the reward axis, every method is scored against a reward floor $\rho$. Rather than choosing $\rho$ arbitrarily, we set it to the base model's median reward over all prompts. We report worst-candidate reward $\min_k r_k$, averaged over prompts, since a batch satisfies the floor $\rho$ precisely when $\min_k r_k \ge \rho$; we also report mean reward. We define the satisfaction rate $\mathrm{SAT}_\rho$ as the fraction of generated candidates whose reward is at least $\rho$:
\begin{equation}
\label{eq:satrate}
  \mathrm{SAT}_\rho \;=\; \frac{1}{N\cdot K}\sum_{i=1}^{N\cdot K} \mathbf{1}\!\left[\,r_i \ge \rho\,\right],
\end{equation}
where $r_i$ is the reward of the $i$-th generated candidate and $\mathbf{1}[\cdot]$ is the indicator function. $N$ is the number of prompts, with $K$ candidates per prompt. Since $\rho$ is the median reward of the base model's outputs, $\mathrm{SAT}_\rho\!=\!0.5$ indicates parity with the base model. On the diversity axis, each batch is measured by the mean pairwise DreamSim~\citep{fu2023dreamsim} and LPIPS~\citep{zhang2018unreasonable} among its $K$ candidates. We additionally report PickScore~\citep{kirstain2023pickapic} as a held-out reward, which indicates whether reward gains reflect a broader improvement in preference.

\subsection{\method Traverses the Satisfaction--Diversity Frontier}
\label{ssec:exp-frontier}

Sweeping the tolerance $\Delta$ of our satisficing objective (Eq.~\ref{eq:obj}) traces a frontier curve between worst-candidate reward and diversity. As shown in Figure~\ref{fig:corner} on FLUX and SANA, baseline methods tend to exhibit either high worst-candidate reward but low diversity (DAS, VASR), or low worst-candidate reward and high diversity (NegToMe). Although varying FK steering's reward-tilt parameter \(\lambda\) also traces a reward--diversity curve, \method's curve Pareto-dominates~\citep{hwang2012multiple} FK steering's over their common DreamSim ranges. At matched DreamSim, \method has higher worst-candidate reward at every overlapping point, with improvements of up to \(0.43\) on FLUX using HPSv3 and \(0.70\) on SANA using ImageReward.

\begin{figure}[!htbp]
\centering
\includegraphics[width=0.49\textwidth]{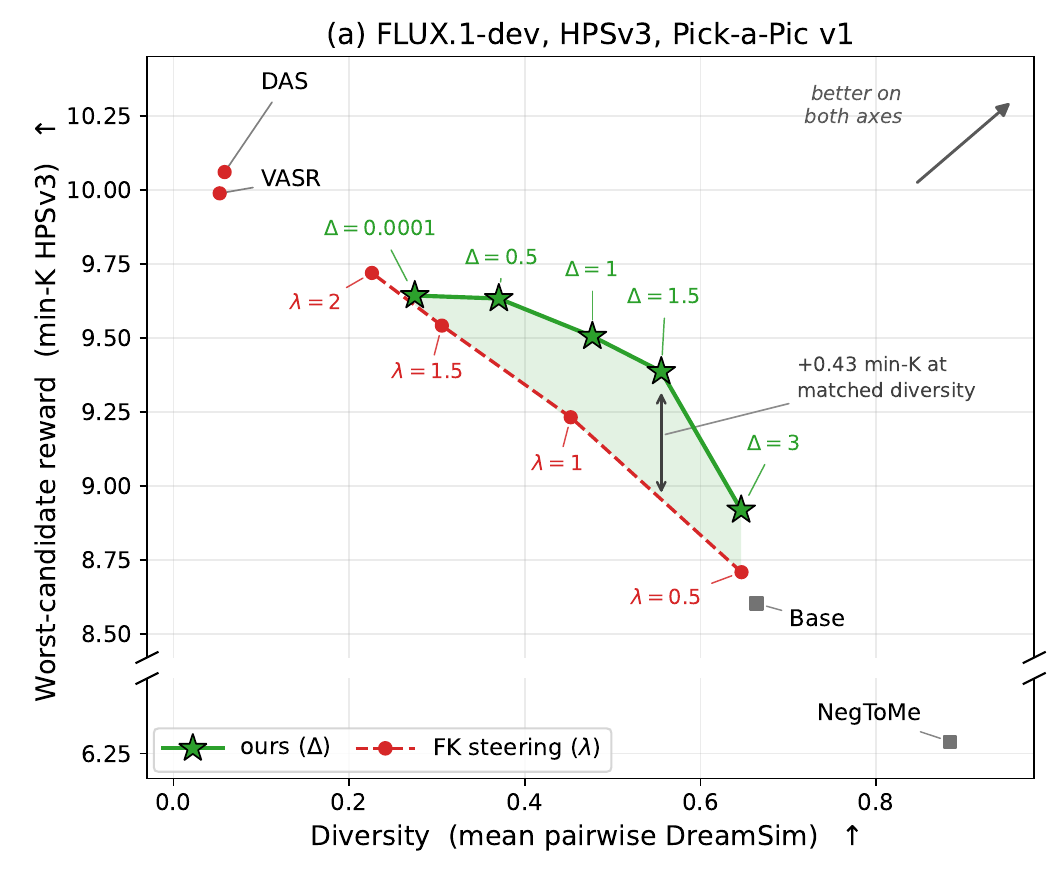}
\hfill
\includegraphics[width=0.49\textwidth]{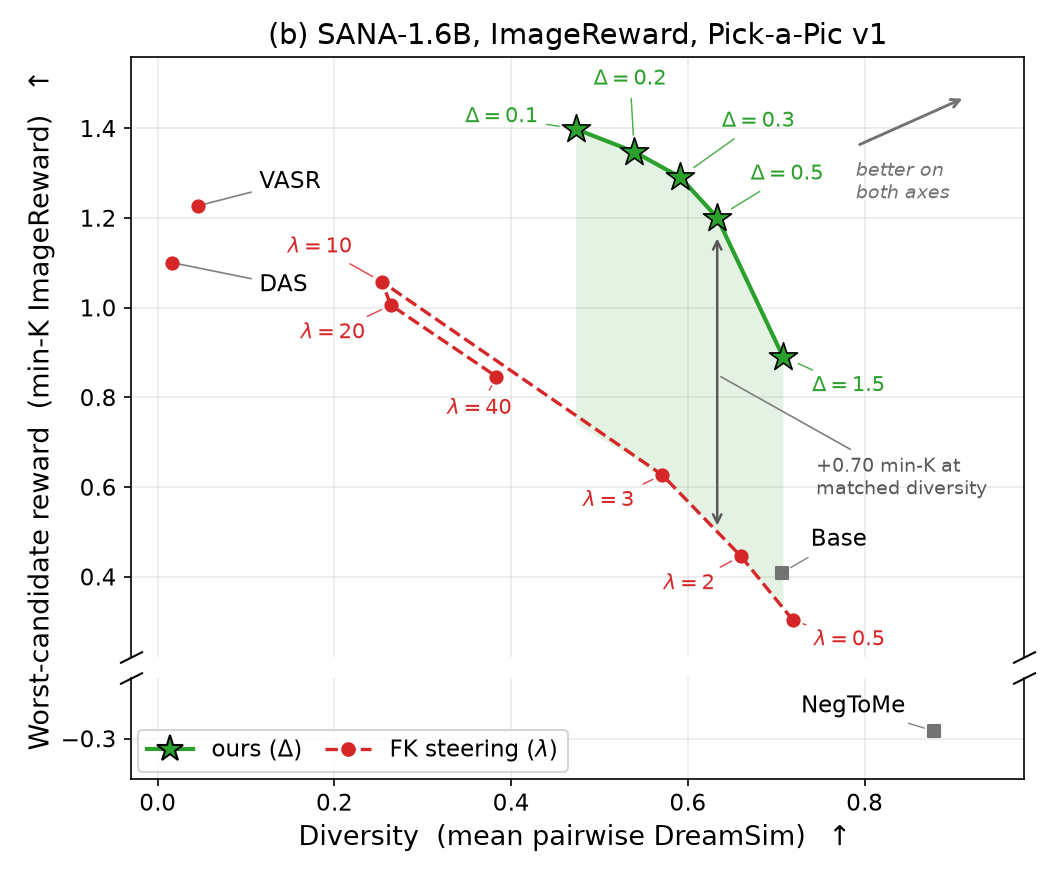}
\vspace{-2mm}
\small
\caption{\textbf{\method achieves better worst-candidate reward at matched diversity versus FK steering.} We plot mean worst-candidate reward $\min_k r_k$ against mean pairwise DreamSim, on Pick-a-Pic. We sweep \method and FK steering on two settings. \textbf{(a)} FLUX.1-dev is the base model and HPSv3 is the reward, with $\Delta$ from $0.0001$ to $3$ and $\lambda$ from $0.5$ to $2$. \textbf{(b)} SANA-1.6B is the base model and ImageReward is the reward, with $\Delta$ from $0.1$ to $1.5$ and $\lambda$ from $0.5$ to $40$. The shaded band represents the worst-candidate reward gap between the two curves at the same diversity, marked at $\Delta{=}1.5$ in (a) where it is $\approx\!0.43$ on FLUX HPSv3, and at $\Delta{=}0.5$ in (b) where it is $\approx\!0.70$ on SANA ImageReward. We note that the top right region, representing high reward and high diversity, is ideal. DAS and VASR collapse to diversity below $0.06$ in both settings, and NegToMe is the most diverse but its worst candidate on average falls below the base model's. \method traces a better empirical frontier than FK steering between these extremes.}
\label{fig:corner}
\end{figure}

\begin{table}[!htbp]
\centering
\small
\caption{\textbf{On Pick-a-Pic, \method is competitive on satisfaction rate without collapsing diversity.} We use FLUX.1-dev with HPSv3 as the reward. DAS has the highest satisfaction rate at $0.612$ but a mean pairwise DreamSim of only $0.059$, while NegToMe is the most diverse at $0.885$ and has the lowest satisfaction rate at $0.276$. \method attains a satisfaction rate of $0.552$ at a DreamSim of $0.555$, and matches base FLUX on both held-out rewards. We use $\rho{=}9.66$; $\pm$ denotes 95\% CIs.}
\vspace{1mm}
\label{tab:flux-pickapic}
\setlength{\tabcolsep}{4pt}
\begin{tabular*}{\textwidth}{@{\extracolsep{\fill}}l ccc cc cc@{}}
\toprule
& \multicolumn{3}{c}{Satisfaction--diversity} & \multicolumn{2}{c}{Additional metrics} & \multicolumn{2}{c}{Held-out} \\
\cmidrule(lr){2-4}\cmidrule(lr){5-6}\cmidrule(lr){7-8}
Method & min $\uparrow$ & DreamSim $\uparrow$ & $\mathrm{SAT}_\rho$ $\uparrow$ & mean $\uparrow$ & LPIPS $\uparrow$ & ImageReward & PickScore \\
\midrule
Base               & $8.60$ & $0.663${\scriptsize$\,\pm0.013$} & $0.500${\scriptsize$\,\pm0.038$} & $9.65$ & $0.643${\scriptsize$\,\pm0.008$} & $1.02$ & $22.28$ \\
NegToMe            & $6.30$ & $0.885${\scriptsize$\,\pm0.009$} & $0.276${\scriptsize$\,\pm0.027$} & $8.17$ & $0.788${\scriptsize$\,\pm0.004$} & $0.79$ & $21.49$ \\
\midrule
FK steering        & $9.54$ & $0.306${\scriptsize$\,\pm0.021$} & $0.570${\scriptsize$\,\pm0.041$} & $10.07$ & $0.291${\scriptsize$\,\pm0.021$} & $1.09$ & $22.37$ \\
FK + NegToMe       & $7.69$ & $0.360${\scriptsize$\,\pm0.017$} & $0.333${\scriptsize$\,\pm0.034$} & $8.70$ & $0.310${\scriptsize$\,\pm0.016$} & $0.99$ & $21.68$ \\
VASR               & $9.99$ & $0.053${\scriptsize$\,\pm0.004$} & $0.598${\scriptsize$\,\pm0.042$} & $10.17$ & $0.033${\scriptsize$\,\pm0.005$} & $1.10$ & $22.40$ \\
DAS                & $10.06$ & $0.059${\scriptsize$\,\pm0.007$} & $0.612${\scriptsize$\,\pm0.041$} & $10.26$ & $0.040${\scriptsize$\,\pm0.007$} & $1.09$ & $22.38$ \\
\midrule
\method (Ours)      & $9.39$ & $0.555${\scriptsize$\,\pm0.01$} & $0.552${\scriptsize$\,\pm0.039$} & $10.03$ & $0.531${\scriptsize$\,\pm0.015$} & $1.04$ & $22.32$ \\
\bottomrule
\end{tabular*}
\end{table}

\begin{table}[!htbp]
\vspace{-2mm}
\centering
\small
\caption{\textbf{On HPDv2, \method is competitive on satisfaction rate without collapsing diversity.} We use FLUX.1-dev with HPSv3 as the reward. \method attains a satisfaction rate of $0.553$ at a mean pairwise DreamSim of $0.567$, against DAS's $0.580$ at $0.082$, and matches base FLUX on both held-out rewards. $\mathrm{SAT}_\rho$ uses floor $\rho{=}13.93$. $\pm$ denotes 95\% confidence intervals.}
\label{tab:flux}
\vspace{1mm}
\setlength{\tabcolsep}{4pt}
\begin{tabular*}{\textwidth}{@{\extracolsep{\fill}}l ccc cc cc@{}}
\toprule
& \multicolumn{3}{c}{Satisfaction--diversity} & \multicolumn{2}{c}{Additional metrics} & \multicolumn{2}{c}{Held-out} \\
\cmidrule(lr){2-4}\cmidrule(lr){5-6}\cmidrule(lr){7-8}
Method & min $\uparrow$ & DreamSim $\uparrow$ & $\mathrm{SAT}_\rho$ $\uparrow$ & mean $\uparrow$ & LPIPS $\uparrow$ & ImageReward & PickScore \\
\midrule
Base               & $12.34$ & $0.655${\scriptsize$\,\pm0.004$} & $0.500${\scriptsize$\,\pm0.015$} & $13.36$ & $0.646${\scriptsize$\,\pm0.003$} & $1.11$ & $22.73$ \\
NegToMe            & $9.65$ & $0.865${\scriptsize$\,\pm0.003$} & $0.309${\scriptsize$\,\pm0.011$} & $12.02$ & $0.781${\scriptsize$\,\pm0.002$} & $0.91$ & $21.95$ \\
\midrule
FK steering        & $13.44$ & $0.197${\scriptsize$\,\pm0.006$} & $0.565${\scriptsize$\,\pm0.016$} & $13.77$ & $0.182${\scriptsize$\,\pm0.007$} & $1.17$ & $22.79$ \\
FK + NegToMe       & $11.53$ & $0.312${\scriptsize$\,\pm0.005$} & $0.320${\scriptsize$\,\pm0.013$} & $12.50$ & $0.277${\scriptsize$\,\pm0.005$} & $1.06$ & $22.04$ \\
VASR               & $13.62$ & $0.053${\scriptsize$\,\pm0.002$} & $0.569${\scriptsize$\,\pm0.017$} & $13.78$ & $0.033${\scriptsize$\,\pm0.002$} & $1.17$ & $22.79$ \\
DAS                & $13.64$ & $0.082${\scriptsize$\,\pm0.003$} & $0.580${\scriptsize$\,\pm0.017$} & $13.82$ & $0.066${\scriptsize$\,\pm0.004$} & $1.16$ & $22.80$ \\
\midrule
\method (Ours)      & $13.07$ & $0.567${\scriptsize$\,\pm0.006$} & $0.553${\scriptsize$\,\pm0.015$} & $13.71$ & $0.559${\scriptsize$\,\pm0.006$} & $1.14$ & $22.77$ \\
\bottomrule
\end{tabular*}
\end{table}

\begin{samepage}
On Pick-a-Pic, Table~\ref{tab:flux-pickapic} shows that \method improves mean HPSv3 reward over base FLUX ($10.03$ against $9.65$). Although DAS and VASR achieve slightly higher mean rewards ($10.26$ and $10.17$ compared to \method's $10.03$), both are markedly less diverse than \method. The same satisfaction--diversity pattern holds on HPDv2 in Table~\ref{tab:flux}, where \method achieves a competitive satisfaction rate and mean HPSv3 reward with notably higher diversity than the reward-steering methods. In particular, adding NegToMe to FK steering increases diversity but reduces the satisfaction rate below the base model's $0.5$, whereas \method maintains above-base satisfaction at near-base diversity. Appendix~\ref{app:experiments} evaluates \method across additional settings and provides ablation and computational analyses.

\end{samepage}

\paragraph{Generalization to Other Base Models.} To show \method is not specific to one architecture or reward model, we extend to SANA-1.6B (flow matching, Table~\ref{tab:sana-ir}) and SD1.5 (UNet, Appendix~\ref{app:experiments} Table~\ref{tab:sd15-ir}) with ImageReward~\citep{xu2023imagereward}. On SANA, \method at $\Delta\!=\!0.3$ achieves the highest worst-candidate reward and satisfaction rate of any evaluated method, with $1.29$ worst-candidate reward versus VASR's $1.23$ and $0.735$ satisfaction rate versus VASR's $0.651$. \method also achieves a DreamSim of $0.591$, versus VASR's $0.046$.

\begin{table}[!htbp]
\centering
\small
\caption{\textbf{\method exceeds every reward-steering method on worst-candidate reward and diversity on SANA-1.6B.} We use Pick-a-Pic with ImageReward; the reward floor is $\rho\!=\!1.19$. Our method is shown at three $\Delta$ settings, tracing its satisfaction--diversity frontier as reflected in Figure~\ref{fig:corner}b. At \(\Delta=0.3\), \method attains a worst-candidate reward of \(1.29\) and DreamSim of \(0.591\), exceeding every reward-steering method on both metrics. $\pm$ denotes 95\% confidence intervals.}
\label{tab:sana-ir}
\vspace{1mm}
\small
\setlength{\tabcolsep}{3.9pt}
\begin{tabular*}{\textwidth}{@{\extracolsep{\fill}}l ccc cc cc@{}}
\toprule
& \multicolumn{3}{c}{Satisfaction--diversity} & \multicolumn{2}{c}{Additional metrics} & \multicolumn{2}{c}{Held-out} \\
\cmidrule(lr){2-4}\cmidrule(lr){5-6}\cmidrule(lr){7-8}
Method & min $\uparrow$ & DreamSim $\uparrow$ & $\mathrm{SAT}_\rho$ $\uparrow$ & mean $\uparrow$ & LPIPS $\uparrow$ & HPSv3 & PickScore \\
\midrule
Base                  & $0.41$ & $0.706${\scriptsize$\,\pm0.012$} & $0.500${\scriptsize$\,\pm0.037$} & $0.91$ & $0.644${\scriptsize$\,\pm0.010$} & $7.53$ & $21.93$ \\
NegToMe               & $-0.29$ & $0.879${\scriptsize$\,\pm0.011$} & $0.380${\scriptsize$\,\pm0.033$} & $0.55$ & $0.743${\scriptsize$\,\pm0.009$} & $5.98$ & $21.29$ \\
\midrule
FK steering           & $1.06$ & $0.254${\scriptsize$\,\pm0.022$} & $0.625${\scriptsize$\,\pm0.041$} & $1.18$ & $0.204${\scriptsize$\,\pm0.021$} & $7.98$ & $22.06$ \\
FK + NegToMe          & $0.84$ & $0.386${\scriptsize$\,\pm0.025$} & $0.585${\scriptsize$\,\pm0.041$} & $1.04$ & $0.278${\scriptsize$\,\pm0.022$} & $7.07$ & $21.63$ \\
VASR                  & $1.23$ & $0.046${\scriptsize$\,\pm0.003$} & $0.651${\scriptsize$\,\pm0.041$} & $1.25$ & $0.023${\scriptsize$\,\pm0.003$} & $8.01$ & $22.09$ \\
DAS                   & $1.10$ & $0.016${\scriptsize$\,\pm0.001$} & $0.599${\scriptsize$\,\pm0.043$} & $1.11$ & $0.004${\scriptsize$\,\pm0.001$} & $7.72$ & $21.93$ \\
\midrule
\method ($\Delta{=}0.3$)  & $1.29$ & $0.591${\scriptsize$\,\pm0.017$} & $0.735${\scriptsize$\,\pm0.037$} & $1.39$ & $0.547${\scriptsize$\,\pm0.016$} & $7.74$ & $21.97$ \\
\method ($\Delta{=}0.5$)  & $1.20$ & $0.633${\scriptsize$\,\pm0.015$} & $0.696${\scriptsize$\,\pm0.038$} & $1.33$ & $0.582${\scriptsize$\,\pm0.014$} & $7.67$ & $21.95$ \\
\method ($\Delta{=}0.75$) & $1.13$ & $0.658${\scriptsize$\,\pm0.015$} & $0.674${\scriptsize$\,\pm0.039$} & $1.28$ & $0.600${\scriptsize$\,\pm0.013$} & $7.63$ & $21.92$ \\
\bottomrule
\end{tabular*}
\end{table}

We further test whether stopping reward maximization once a candidate satisfies the floor, without an explicit diversity term, is sufficient to preserve diversity. We therefore cap DAS's reward at \(\rho\) while leaving its other settings unchanged. This reward capping increases DAS's DreamSim from \(0.016\) to \(0.066\), but the resulting diversity remains far below \method's \(0.658\); the change in satisfaction rate is not statistically significant. Thus, reward capping alone does not account for \method's diversity (Appendix~\ref{ssec:app-dasclip}).

\subsection{Qualitative Results}
\label{ssec:exp-qualitative}

\begin{figure}[H]
\centering
\includegraphics[width=\textwidth]{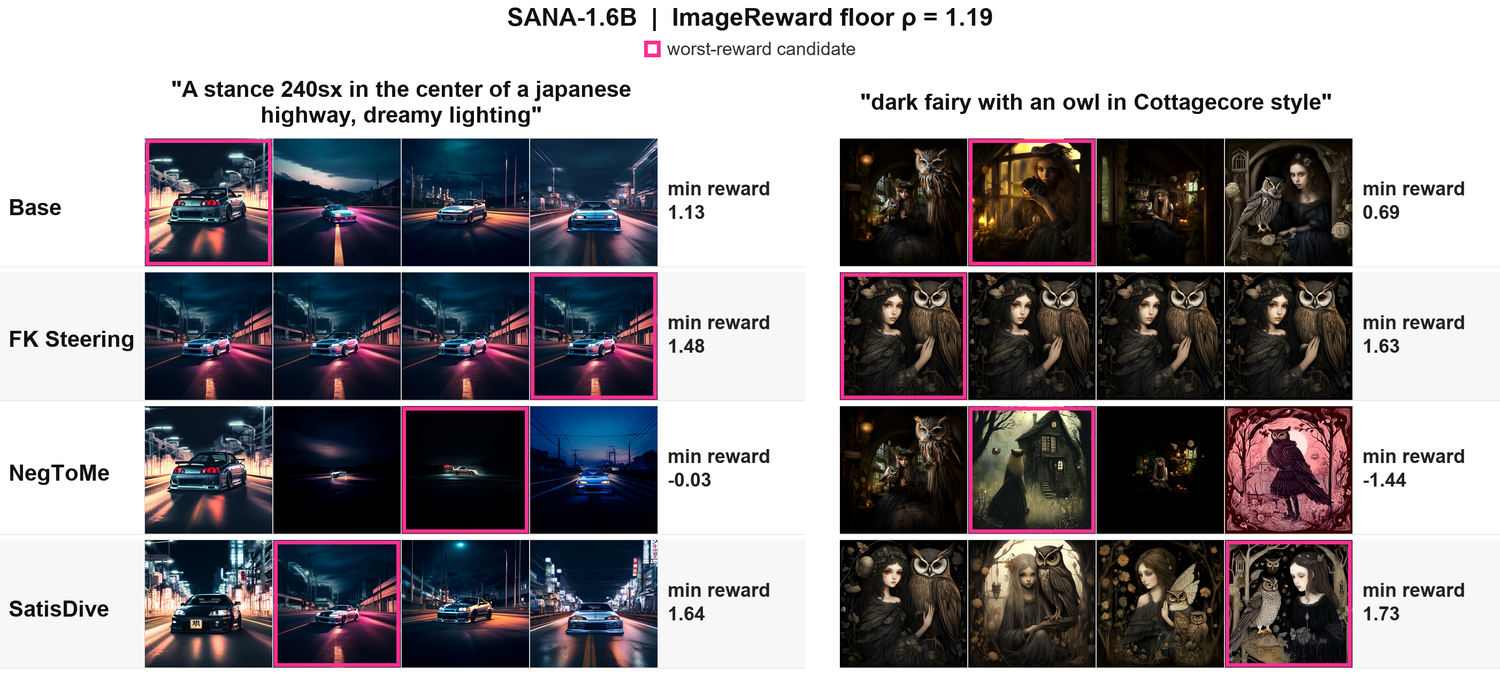}
\vspace{-6mm}
\small
\caption{\textbf{On two Pick-a-Pic prompts, only \method returns four visually distinct images that are all above the reward floor.} We use SANA-1.6B with ImageReward and matched starting seeds; the reward floor is $\rho{=}1.19$. The magenta border marks the worst-reward candidate in each row, whose reward is reported on the right. Base and NegToMe produce visually varied batches, but their worst-reward candidates do not satisfy the reward floor. NegToMe's worst-reward car image (on the left) is nearly black, and its worst-reward fairy image (on the right) does not clearly show an owl. FK steering's images exceed the reward floor, but its four images repeat nearly the same composition.}
\label{fig:sana-qual}
\end{figure}

\begin{samepage}
Figure~\ref{fig:sana-qual} shows individual batches from the SANA evaluation summarized in Figure~\ref{fig:corner}. FK steering returns four images above the reward floor but repeats nearly the same composition, while Base and NegToMe produce more varied batches in which some images fall below the floor. On both prompts, \method is the only method that returns four above-floor images with visibly different compositions. Additional qualitative results appear in Appendix~\ref{app:qualitative}.

\end{samepage}

\FloatBarrier

\section{Related Work}
\label{sec:related}

\textbf{Reward-driven sampling.} Existing methods increase reward through resampling~\citep{singhal2025fk,wu2023tds,kim2025das,shekhar2026vasr}, value-based lookahead~\citep{li2025svdd,jain2025dts}, stochastic transitions~\citep{holderrieth2026glass}, adaptive model evaluation~\citep{kim2026inferencetime}, or noise-space optimization and regularization~\citep{hwang2026ntrk,tang2025inferencetime,eyring2024reno,harrington2026never,zhai2025mira}. These methods do not explicitly optimize batch diversity and may collapse candidate lineages~\citep{shekhar2026vasr}. CREPE~\citep{he2026crepe} uses replica exchange to preserve diversity under reward tilting (comparison in Appendix~\ref{ssec:app-crepe}), while SGI~\citep{parmar2026scaling} selects a reward--diversity batch from a large candidate pool; neither imposes a reward floor on every candidate.

\Needspace{4\baselineskip}
\textbf{Reward-free diversity.} NegToMe~\citep{singh2024negtome}, CNO~\citep{kim2026diverse}, Particle Guidance~\citep{corso2023particle}, and EDDY~\citep{vinograd2026eddy} repel candidates in feature or prediction space. Other methods vary conditioning, intermediate features, or initial noise~\citep{sadat2024cads,yadav2026stride,li2026divin}, while general guidance modifies model predictions, attention, or guidance schedules~\citep{ho2021classifierfree,ahn2024self,hong2024smoothed,karras2024guiding,kynkaanniemi2024applying,yu2023freedom}. These methods increase diversity without imposing a reward floor.

\textbf{Pareto and satisficing objectives.} AIG~\citep{jena2025elucidating} studies average reward against distributional coverage; PROUD~\citep{yao2024proud} defines Pareto objectives over per-image properties; and ParetoSlider~\citep{golan2026paretoslider}, constrained diffusion alignment~\citep{khalafi2026composition}, and SITAlign~\citep{chehade2025bounded} optimize expected rewards. In contrast, our training-free formulation imposes a reward floor on every candidate and a diversity requirement on the batch, producing a frontier between worst-candidate reward and batch diversity.

\section{Conclusion}\label{sec:conclusion}

\method formulates image generation as satisficing: every candidate must satisfy a reward floor, and the batch must satisfy a diversity cutoff. Varying the floor provably traces a satisfaction--diversity Pareto frontier. On FLUX.1-dev with HPSv3 and SANA-1.6B with ImageReward, \method's curves Pareto-dominate FK steering's over their overlapping DreamSim ranges, improving worst-candidate reward at matched DreamSim by up to $0.43$ and $0.70$, respectively. Extending the formulation to multiple reward floors remains future work.

\clearpage
\appendix
\addtocontents{toc}{\protect\setcounter{tocdepth}{2}}
\begingroup
\renewcommand{\contentsname}{Appendix Contents}
\tableofcontents
\endgroup
\clearpage
\section{Notation} 
\label{app:notation}
We summarize the notation used throughout the paper.

\begin{table}[h]
\centering
\caption{Summary of Notation.}
\vspace{2mm}
\label{tab:notation}
\small
\begin{tabular}{@{}lp{0.76\linewidth}@{}}
\toprule
\textbf{Symbol} & \textbf{Description} \\
\midrule
    $c$ & Text prompt \\
    $K$ & Number of candidates in a batch \\
    $k$ & Candidate index \\
    $T$ & Number of sampling steps \\
    $s$ & Discrete denoising-step index, $s=1,\ldots,T$ \\
    $t$ & Model time, decreasing from $1$ (noise) to $0$ (image) \\
    $x_t^k$ & Latent of candidate $k$ at time $t$ \\
    $x_0^{1:K}$ & Batch of $K$ generated images \\
    $\hat x_0^k$ & Model's current estimate of candidate $k$'s final image \\
    $p_\theta(x_0\mid c)$ & Output distribution of the frozen base text-to-image model \\
    $v_\theta(x_t,t,c)$ & Velocity field of a flow-based text-to-image model \\
\midrule
    $r(x_0,c)$ & Learned reward for image $x_0$ and prompt $c$ \\
    $r_k$ & Reward of candidate $k$ \\
    $\rho$ & Reward floor used to define satisfaction \\
    $\tau$ & Batch-relative reward cutoff, $\tau=\max_k r_k-\Delta$ \\
    $\Delta$ & Tolerance controlling the gap between the largest reward and $\tau$ \\
    $q_k$ & Smooth gate indicating whether candidate $k$ exceeds $\tau$ \\
    $d_{ij}$ & Cosine distance between the features of candidates $i$ and $j$ \\
    $D$ & Reward-gated mean pairwise distance; in the hard limit, this is the unweighted mean on admissible batches \\
    $\delta$ & Diversity cutoff \\
    $\beta_r$ & Reward temperature \\
    $\beta_d$ & Diversity temperature \\
    $\mathcal J_r$ & Reward penalty \\
    $\mathcal J_d$ & Diversity penalty \\
    $\mathcal J$ & Sum of the reward and diversity penalties \\
\midrule
    $\mathcal G$ & Sampling-step indices at which candidates are scored and updated using $\mathcal J$ \\
    $r_k^{\mathrm{prev}}$ & Reward of candidate $k$ at the previous update step \\
    $m$ & Number of update steps remaining \\
    $\tilde r_k$ & Extrapolated reward of candidate $k$ \\
\midrule
    $p_\lambda$ & Reward-tilted output distribution with tilt strength $\lambda$ \\
    $\pi^\star$ & Hard satisficing target distribution over batches \\
    $A_\rho$ & Images clearing reward floor $\rho$ \\
    $p_\rho$ & Base output distribution restricted to $A_\rho$ \\
\bottomrule
\end{tabular}
\end{table}

\newpage
\begingroup
\allowdisplaybreaks[1]
\setlength{\abovedisplayskip}{5pt plus2pt minus1pt}
\setlength{\belowdisplayskip}{5pt plus2pt minus1pt}
\setlength{\abovedisplayshortskip}{3pt plus1pt}
\setlength{\belowdisplayshortskip}{4pt plus1pt}
\newcommand{\proofstep}[1]{\par\addvspace{0.6\baselineskip}\Needspace{3\baselineskip}\noindent\textbf{#1}\enspace\ignorespaces}
\section{Theoretical Results and Proofs}
\label{app:theory}

This section first analyzes reward tilting, then characterizes the reward-floor frontier, and finally proves the zero-temperature limit of \method's objective.

Fix a prompt $c$. Let $p=p_\theta(\cdot\mid c)$ denote the base-model distribution, and write $R(x)=r(x,c)$ for the reward. Rewards are measurable and real-valued. All image distributions below have densities with respect to the same reference measure, denoted by $dx$ in integrals. We use the same symbol for a distribution and its density. For a measurable set $E$, its probability is $p(E)=\int_Ep(x)dx$, and $E^c$ denotes its complement. An indicator $\mathbf1_E(x)$ is one on $E$ and zero elsewhere.

We assume that the batch size $K\ge2$ is a fixed finite integer and write $x^{1:K}=(x^1,\ldots,x^K)$. Candidate indices satisfy $k,\ell\in\{1,\ldots,K\}$, and pair indices satisfy $1\le i<j\le K$. For a real reward floor $\rho$, let $A_\rho=\{x:R(x)\ge\rho\}$ be the set of images satisfying the floor. Its base-model probability is $Z_\rho=p(A_\rho)$. A floor is feasible when $Z_\rho>0$. For such a floor, the base distribution restricted to $A_\rho$ is
\begin{equation}
\label{eq:app-floor-density}
  p_\rho(x)=\frac{p(x)\mathbf 1[x\in A_\rho]}{Z_\rho}.
\end{equation}

\subsection{Reward Tilting}

\paragraph{Precise conditions for Proposition~\ref{prop:tilt}.}
At a given resampling step, fix $K$ candidates $x^{1:K}$ with rewards $r_k=R(x^k)$, and let $j^\star$ be the index of the unique highest-reward candidate. Let $\lambda\ge0$ be the reward-tilt strength and $a_k>0$ a finite factor in candidate $k$'s resampling mass that does not depend on $\lambda$; setting every $a_k=1$ recovers the pure exponential reward weighting in Eq.~\ref{eq:tilt}. Define the normalized resampling mass of candidate $k$ as
\[
  \omega_k(\lambda)=\frac{a_k e^{\lambda r_k}}{\sum_{\ell=1}^K a_\ell e^{\lambda r_\ell}}.
\]
Let $N_k\in\{0,\ldots,K\}$ be the number of copies of candidate $k$ after resampling. We assume that resampling returns $K$ candidates from this fixed pool, so $\sum_kN_k=K$, and that $\mathbb E[N_k]=K\omega_k(\lambda)$ for every $k$. We do not assume that the selections are independent.

\begin{proof}[Proof of Proposition~\ref{prop:tilt}]
We first show that the highest-reward candidate receives resampling mass tending to one. We then bound the probability that any other candidate survives among the $K$ outputs. All probabilities and expectations concern only the randomness of resampling from the fixed pool above.

\proofstep{Step 1: The highest-reward candidate's mass tends to one.} For any $\ell\ne j^\star$, the shared normalization cancels:
\begin{equation}
\label{eq:app-resampling-mass-ratio}
  \frac{\omega_\ell(\lambda)}{\omega_{j^\star}(\lambda)}
  =\frac{a_\ell e^{\lambda r_\ell}}{a_{j^\star}e^{\lambda r_{j^\star}}}
  =\frac{a_\ell}{a_{j^\star}}e^{-\lambda(r_{j^\star}-r_\ell)}
  \longrightarrow 0.
\end{equation}
The limit holds because $r_{j^\star}-r_\ell>0$. Dividing the numerator and denominator of $\omega_{j^\star}$ by $a_{j^\star}e^{\lambda r_{j^\star}}$ gives
\[
\omega_{j^\star}(\lambda)=\left(1+\sum_{\ell\ne j^\star}\frac{a_\ell}{a_{j^\star}}e^{-\lambda(r_{j^\star}-r_\ell)}\right)^{-1}\longrightarrow1.
\]
Each summand decreases to zero, so the denominator decreases to one. Its reciprocal and the expected copy count $K\omega_{j^\star}(\lambda)$ therefore increase.

\proofstep{Step 2: All $K$ outputs are copies with probability tending to one.} Because resampling returns exactly $K$ candidates, $N_{j^\star}<K$ if and only if some other candidate survives. Also, $\mathbf1[N_\ell\ge1]\le N_\ell$, so taking expectations gives $\Pr[N_\ell\ge1]\le\mathbb E[N_\ell]$. The union bound and the expected copy counts therefore give
\begin{align*}
\Pr[N_{j^\star}<K]
&=\Pr\!\left[\bigcup_{\ell\ne j^\star}\{N_\ell\ge1\}\right]
\le\sum_{\ell\ne j^\star}\Pr[N_\ell\ge1]\\
&\le\sum_{\ell\ne j^\star}\mathbb E[N_\ell]
=K\sum_{\ell\ne j^\star}\omega_\ell
=K(1-\omega_{j^\star})\longrightarrow0.
\end{align*}
The last equality uses $\sum_k\omega_k=1$, and the limit uses $\omega_{j^\star}\to1$. Taking the complementary event gives $\Pr[N_{j^\star}=K]=1-\Pr[N_{j^\star}<K]\to1$.

Multinomial, Srinivasan Sampling Process (SSP), and systematic resampling satisfy the expectation condition. The conclusion also holds if VASR-MAX's redirection rule is applied after this systematic resampling: when resampling returns $K$ copies of the highest-reward candidate, the redirection rule does not change the result.
\end{proof}

For the population-level comparison, assume $R$ is bounded above and let $Z(\lambda)=\mathbb E_p[e^{\lambda R}]$. The normalized tilt in Eq.~\eqref{eq:tilt} is $p_\lambda(x)=p(x)e^{\lambda R(x)}/Z(\lambda)$.

\begin{corollary}[A finite reward tilt does not enforce the reward floor]
\label{cor:tiltdom}
Fix $\rho$ such that $0<p(A_\rho)<1$. Then $p_\lambda(A_\rho)<1$ for every finite $\lambda\ge0$.
\end{corollary}

\begin{proof}
Let $B$ be a finite upper bound for $R$. Since $0<e^{\lambda R}\le e^{\lambda B}$ and $\int p(x)dx=1$, taking expectations gives $0<Z(\lambda)\le e^{\lambda B}<\infty$. The tilt therefore preserves positive probability below the floor:
\[
p_\lambda(A_\rho^c)=\frac{\int_{A_\rho^c}e^{\lambda R(x)}p(x)dx}{Z(\lambda)}>0,
\qquad p_\lambda(A_\rho)=1-p_\lambda(A_\rho^c)<1.
\]
\end{proof}

\subsection{Reward-Floor Frontier}

In this subsection, support means the set on which a density is positive, with sets of zero base-model probability ignored; $\operatorname{supp}(\nu)$ denotes the support of a distribution $\nu$. A restriction is proper when it removes a set with positive base-model probability. The theorem's diversity assumption compares distributions by inclusion of these supports, including distributions that assign different probabilities to the same support. In particular, equal supports have equal diversity: each support is contained in the other, so both diversity inequalities hold. This distribution-level diversity is distinct from the per-batch statistic $D$ used in the objective.

We use $\mathrm{KL}(\nu\Vert p)=\int\nu(x)\log(\nu(x)/p(x))dx$, with zero-density contributions taken as zero. This divergence is $+\infty$ if $\nu$ assigns positive probability to any set of zero $p$ probability. Thus finite divergence requires $\nu$ to be supported within $p$. All support comparisons and equalities between distributions below ignore sets of zero base-model probability.

The factorization in Eq.~\eqref{eq:factorize} follows by substituting Eq.~\eqref{eq:app-floor-density} for each candidate:
\[
\prod_{k=1}^K\left(p(x^k)\mathbf1[x^k\in A_\rho]\right)=\prod_{k=1}^K\left(Z_\rho p_\rho(x^k)\right)=Z_\rho^K\prod_{k=1}^K p_\rho(x^k).
\]
Multiplying both sides by the batch diversity indicator in Eq.~\eqref{eq:constraint} preserves this equality. Since $Z_\rho^K>0$ is constant with respect to $x^{1:K}$, it is absorbed into the batch normalizer. The diversity indicator remains in the target, exactly as in Eq.~\eqref{eq:factorize}.

The following lemma records a standard conditioning property of KL divergence; see \citet{csiszar2004information} for background on information projections.

\begin{lemma}[Among distributions over $A_\rho$, $p_\rho$ is uniquely closest to the base model]
\label{lem:iproj}
For a feasible floor $\rho$, the distribution $p_\rho$ uniquely minimizes $\mathrm{KL}(\cdot\Vert p)$ among distributions that assign probability one to $A_\rho$. Moreover, every such distribution $\nu$ with finite $\mathrm{KL}(\nu\Vert p)$ has support contained in $\operatorname{supp}(p_\rho)$.
\end{lemma}

\begin{proof}
We compute the KL divergence of $p_\rho$, identify where a competing distribution can assign probability, and then compare their KL divergences directly.

\proofstep{Step 1: Compute the reference KL and establish support containment.} Wherever $p_\rho$ is positive, $p_\rho/p=1/Z_\rho$. Substituting this ratio and using $\int_{A_\rho}p_\rho(x)dx=1$ gives
\begin{equation}
\mathrm{KL}(p_\rho\Vert p)=\int_{A_\rho}p_\rho(x)\log(1/Z_\rho)dx=-\log Z_\rho\int_{A_\rho}p_\rho(x)dx=-\log Z_\rho<\infty.
\label{eq:app-iproj-reference}
\end{equation}
An infinite-KL distribution cannot improve this finite value. Consider any $\nu$ with finite KL to $p$ and $\nu(A_\rho)=1$. Finite KL requires its support to lie within that of $p$. The reward constraint also confines it to $A_\rho$. Its support therefore lies within their intersection, which is exactly the support of $p_\rho$.

\proofstep{Step 2: Every competitor has at least as much KL divergence.} The support containment makes the ratios below defined almost everywhere under $\nu$. We insert $p_\rho/p=1/Z_\rho$ into the density ratio and split the logarithm. The constant $\log Z_\rho$ can then be taken outside the integral. Finally, $\int_{A_\rho}\nu(x)dx=1$ gives
\begin{align}
  \mathrm{KL}(\nu\Vert p)
  &=\int_{A_\rho}\nu(x)\log\frac{\nu(x)}{p(x)}dx \notag\\
  &=\int_{A_\rho}\nu(x)\log\left(\frac{\nu(x)}{p_\rho(x)}\frac{1}{Z_\rho}\right)dx \notag\\
  &=\int_{A_\rho}\nu(x)\log\frac{\nu(x)}{p_\rho(x)}dx-\log Z_\rho\int_{A_\rho}\nu(x)dx \notag\\
  &=\mathrm{KL}(\nu\Vert p_\rho)-\log Z_\rho.\label{eq:app-iproj-final}
\end{align}
KL divergence is nonnegative, with equality only for identical distributions. Thus $\mathrm{KL}(\nu\Vert p)\ge-\log Z_\rho=\mathrm{KL}(p_\rho\Vert p)$, with equality only for $\nu=p_\rho$. This proves the unique minimum; support containment was established above.
\end{proof}

\begin{proof}[Proof of Theorem~\ref{thm:frontier}]
The proof has four steps. First, $p_\rho$ retains maximal diversity at each floor. Second, we identify its highest guaranteed reward, which need not equal the specified floor $\rho$. Third, we show that improving this guarantee requires losing diversity. Finally, we show that the family represents every frontier point in the stated comparison class.

\proofstep{Step 1: $p_\rho$ has maximal diversity at each floor.} By Eq.~\eqref{eq:app-floor-density}, $p_\rho$ assigns probability one to $A_\rho$ and is positive exactly where $p$ is positive within $A_\rho$. Let $\rho'\ge\rho$ be another feasible floor. Since $R(x)\ge\rho'$ implies $R(x)\ge\rho$, we have $A_{\rho'}\subseteq A_\rho$. Intersecting these sets with the support of $p$ gives
\begin{equation}
\label{eq:app-frontier-support}
  \operatorname{supp}(p_{\rho'})=\operatorname{supp}(p)\cap A_{\rho'}
  \subseteq\operatorname{supp}(p)\cap A_\rho=\operatorname{supp}(p_\rho).
\end{equation}
By assumption, restricting support cannot increase diversity, and a proper restriction strictly decreases it. Equation~\eqref{eq:app-frontier-support} therefore implies that $p_{\rho'}$ is no more diverse than $p_\rho$. The inequality is strict when $p(A_\rho\setminus A_{\rho'})>0$. Otherwise, the two conditioning sets differ only on a set of zero base-model probability, so $p_{\rho'}=p_\rho$ almost everywhere.

Now consider any distribution that assigns probability one to $A_\rho$ and has finite divergence to $p$. Lemma~\ref{lem:iproj} places its support within $\operatorname{supp}(p_\rho)$. Its diversity therefore cannot exceed the diversity of $p_\rho$. Thus, $p_\rho$ has maximal diversity at each floor.

\proofstep{Step 2: Find the highest guaranteed reward.} To compare satisfaction, we need the highest floor actually guaranteed by $p_\rho$. This can exceed the specified $\rho$ when there is a gap in the reward values. Write $g=\sup\{a\in\mathbb R:p_\rho(R\ge a)=1\}$ for the supremum of the guaranteed floors. Since $\rho$ itself is guaranteed, $g\ge\rho$. We now check that $g$ is finite and is itself guaranteed, so it can be used in the Pareto comparison.

First, $g$ is finite. If it were infinite, every event $\{R\ge n\}$, for positive integers $n$, would have probability one. These events decrease to the empty set because rewards are real-valued. Continuity of probability would then give probability one to the empty set, a contradiction.

Next, the floor $g$ is itself guaranteed. For each positive integer $n$, the supremum definition supplies a guaranteed floor greater than $g-1/n$. This implies $p_\rho(R\ge g-1/n)=1$. These events decrease to $\{R\ge g\}$, so continuity of probability gives
\[
p_\rho(R\ge g)=\lim_{n\to\infty}p_\rho(R\ge g-1/n)=1.
\]
Thus $g$ is the highest floor attained with probability one, not just a supremum of smaller guaranteed floors. The same decreasing-event argument applies to any distribution with finite guaranteed reward.

\proofstep{Step 3: No competitor dominates $p_\rho$.} A distribution dominates another if neither satisfaction nor diversity is lower and at least one is higher. We rule out both ways a competitor could improve on $p_\rho$.

\emph{Greater diversity without lower satisfaction.} Consider a distribution $\nu$ with $\mathrm{KL}(\nu\Vert p)<\infty$ that guarantees reward at least $g$. Since $g\ge\rho$, it satisfies $\nu(A_\rho)=1$. Lemma~\ref{lem:iproj} therefore places its support within that of $p_\rho$, so $\nu$ cannot have greater diversity.

\emph{Higher satisfaction without lower diversity.} Suppose $\nu$ guarantees a strictly higher floor $h>g$. Then $p_\rho(R<h)>0$: otherwise $h$ would also be guaranteed under $p_\rho$, contradicting the definition of $g$. In contrast, $\nu(R<h)=0$. Since $p_\rho$ assigns probability one to $A_\rho$, substituting its density from Eq.~\eqref{eq:app-floor-density} gives
\[
p(A_\rho\setminus A_h)=Z_\rho p_\rho(R<h)>0,
\qquad \nu(A_\rho\setminus A_h)\le\nu(R<h)=0.
\]
The competitor therefore excludes a set of positive base-model probability from the support of $p_\rho$. By strict support monotonicity, $\nu$ has lower diversity. No competitor can improve either guaranteed reward or diversity without reducing the other, so $p_\rho$ is Pareto-optimal. Floors producing the same distribution give the same frontier point.

\proofstep{Step 4: Every frontier point is represented by the family.} Step 3 establishes that each $p_\rho$ is Pareto-optimal. We must also show that no undominated satisfaction--diversity pair is missing from this family.

Take any $\nu$ with finite divergence to $p$ and finite guaranteed reward $h=\sup\{a\in\mathbb R:\nu(R\ge a)=1\}$. The decreasing-event argument in Step 2 gives $\nu(A_h)=1$. We must have $Z_h=p(A_h)>0$: otherwise $\nu$ would assign probability one to a set of zero $p$ probability, contradicting finite KL. Thus $p_h$ exists. Lemma~\ref{lem:iproj} gives $\operatorname{supp}(\nu)\subseteq\operatorname{supp}(p_h)$, so $p_h$ has at least as much diversity as $\nu$.

We also need to check that $p_h$ has the same satisfaction as $\nu$. It guarantees reward at least $h$ by construction. Suppose it guaranteed some floor $b>h$. Its conditional density would give $p(A_h\setminus A_b)=Z_hp_h(R<b)=0$. Finite KL would then imply $\nu(A_h\setminus A_b)=0$. Together with $\nu(A_h)=1$, this gives $\nu(A_b)=1$, contradicting the definition of $h$. Therefore, the highest guaranteed reward of $p_h$ is exactly $h$.

We have shown that $p_h$ has the same satisfaction as $\nu$ and at least as much diversity. A proper support inclusion gives strictly greater diversity, so such a $\nu$ is dominated by $p_h$. Equal supports give the same diversity and hence the same satisfaction--diversity pair. Every undominated pair of finite satisfaction and diversity is therefore represented by the family, although more than one distribution can represent the same pair.

\end{proof}

\subsection{Zero-Temperature Limit of \method's Objective}

\begin{proof}[Proof of Proposition~\ref{prop:hard-limit}]
We first show that the objective's weight approaches the indicator of a feasible batch. We then show that its normalizing constant converges to the probability of that event. Finally, we show that the normalized distribution converges to $\pi^\star$.

\textbf{Setup.} Fix a current candidate batch and let $M$ be its maximum reward. Fix the tolerance $\Delta>0$ and set $\tau=\rho=M-\Delta$. Also fix the diversity cutoff $\delta\in\mathbb R$ for this objective evaluation. These values remain fixed as we integrate over batches and take temperature limits. In particular, $M$ is not recomputed for each integration variable.

For a generic batch $x^{1:K}$, write $r_k=R(x^k)$ for candidate $k$'s reward. The density of a batch of independent base-model samples is $P_c(x^{1:K})=\prod_{k=1}^Kp(x^k)$. Integrals over batches use the product reference measure $dx^1\cdots dx^K$.

Let $d_{ij}=d_{ij}(x^{1:K})$ be the measurable cosine distances between candidates' features. We hold the feature map fixed and assume these distances are well defined. They lie in $[0,2]$ and do not depend on the temperatures. There are $\binom K2=K(K-1)/2$ candidate pairs. For a reward temperature $\beta_r>0$, write the reward gate and weighted diversity as
\[
q_k=\frac{1}{1+e^{(\rho-r_k)/\beta_r}},\qquad D_{\beta_r}=\frac{\sum_{i<j}q_iq_jd_{ij}}{\sum_{i<j}q_iq_j}.
\]
Here $D_{\beta_r}$ is the $D$ in Eq.~\eqref{eq:divpen}. The subscript makes its dependence on the reward temperature explicit. The hard target instead uses the unweighted mean $D_0=\binom K2^{-1}\sum_{i<j}d_{ij}$. We will show below that $D_{\beta_r}\to D_0$ whenever all rewards exceed the floor.

Define the feasible batch set $C=\{x^{1:K}:r_k\ge\rho\ \text{for every }k,\ D_0\ge\delta\}$. The proposition assumes $P_c(C)>0$, so conditioning on this set is possible. It also assumes $P_c(r_k=\rho)=0$ for every $k$ and $P_c(D_0=\delta)=0$. Since there are finitely many candidates, the union of these boundary events also has probability zero. We may therefore establish the limit outside these events.

For temperatures $\beta_r,\beta_d>0$, define the weight $w_{\beta_r,\beta_d}=\exp(-\mathcal J)$ using Eq.~\eqref{eq:obj}. Its normalizing constant is $Z_{\beta_r,\beta_d}=\mathbb E_{P_c}[w_{\beta_r,\beta_d}]$.

These quantities are well defined at every positive temperature. Every $q_k$ lies in $(0,1)$, and $K\ge2$ ensures that there is at least one pair. Thus the denominator of $D_{\beta_r}$ is positive, and its weighted average is finite. Both penalties are consequently finite and nonnegative, giving $0<w_{\beta_r,\beta_d}\le1$. Taking expectations gives $0<Z_{\beta_r,\beta_d}\le1$. The induced distribution is therefore
\begin{equation}
\label{eq:app-induced-distribution}
  \widetilde\pi_{\beta_r,\beta_d}(x^{1:K})
  =\frac{P_c(x^{1:K})w_{\beta_r,\beta_d}(x^{1:K})}{Z_{\beta_r,\beta_d}}.
\end{equation}
\proofstep{Step 1: The weights converge to the feasible-batch indicator.} Equation~\eqref{eq:rewpen} gives $\mathcal J_r=-\sum_k\log q_k$. We substitute this and $\operatorname{softplus}(z)=\log(1+e^z)$ into $\exp(-\mathcal J_r-\mathcal J_d)$. Exponentiating the logarithms gives the quotient below. Rewriting its denominator using $\operatorname{sigmoid}(z)=1/(1+e^{-z})$ gives a product of reward gates and a diversity gate:
\begin{equation}
\label{eq:app-objective-weight}
\begin{aligned}
  w_{\beta_r,\beta_d}
  &=e^{\sum_k\log q_k}\,e^{-\log(1+e^{(\delta-D_{\beta_r})/\beta_d})} \\
  &=\frac{\prod_{k=1}^K q_k}{1+e^{(\delta-D_{\beta_r})/\beta_d}}
  =\left(\prod_{k=1}^K q_k\right)
  \operatorname{sigmoid}\left(\frac{D_{\beta_r}-\delta}{\beta_d}\right).
\end{aligned}
\end{equation}
Take $\beta_r,\beta_d\to0$ through positive values. Outside the boundary events, there are two reward cases.

\emph{Case 1: At least one candidate is below the reward floor.} If $r_k<\rho$, then $e^{(\rho-r_k)/\beta_r}\to\infty$, so $q_k\to0$. All factors in Eq.~\eqref{eq:app-objective-weight} lie between zero and one. Hence $0\le w_{\beta_r,\beta_d}\le q_k\to0$: one failed reward requirement makes the batch weight vanish, regardless of diversity.

\emph{Case 2: Every candidate is above the reward floor.} Now every $r_k>\rho$, so $e^{(\rho-r_k)/\beta_r}\to0$ and every $q_k\to1$. Each pair weight $q_iq_j$ then tends to one. Taking limits in the finite sums, whose denominator tends to $\binom K2>0$, gives
\begin{equation}
\label{eq:app-diversity-limit}
  \lim_{\beta_r\to0}D_{\beta_r}
  =\frac{\sum_{i<j}1\cdot1\cdot d_{ij}}{\sum_{i<j}1\cdot1}
  =\frac{\sum_{i<j}d_{ij}}{\binom K2}=D_0.
\end{equation}
We must now check the diversity gate, whose numerator and temperature both change. Since $D_0\ne\delta$, the limiting diversity has a nonzero gap from the cutoff. Convergence to $D_0$ ensures that, for sufficiently small $\beta_r$, $|D_{\beta_r}-D_0|<|D_0-\delta|/2$. Thus $D_{\beta_r}$ stays on the same side of $\delta$ as $D_0$. The triangle inequality also bounds its distance from the cutoff:
\[
|D_{\beta_r}-\delta|\ge |D_0-\delta|-|D_{\beta_r}-D_0|>\frac{|D_0-\delta|}{2}>0.
\]
The positive lower bound does not depend on either temperature. Dividing by $\beta_d\to0$ therefore makes the magnitude of the sigmoid argument diverge. When $D_0>\delta$, its sign is positive and the diversity gate tends to one. When $D_0<\delta$, its sign is negative and the gate tends to zero. This holds regardless of the relative temperature rates.

Combining these diversity cases with the reward cases gives
\begin{equation}
\label{eq:app-weight-limit}
  \lim_{\beta_r,\beta_d\to0}w_{\beta_r,\beta_d}
  =\left(\prod_k\mathbf1[r_k\ge\rho]\right)\mathbf1[D_0\ge\delta]=\mathbf1_C
  \quad\text{almost everywhere under }P_c.
\end{equation}
\proofstep{Step 2: The normalizer converges to the feasible-batch probability.} To normalize the limiting weights, we also need the limit of their expectation. Since $0\le w_{\beta_r,\beta_d}\le1$ and $P_c$ is a probability distribution, the constant one is an integrable bound. Dominated convergence therefore allows us to move the limit inside the expectation in Eq.~\eqref{eq:app-weight-limit}:
\begin{equation}
\label{eq:app-normalizer-limit}
\lim_{\beta_r,\beta_d\to0}Z_{\beta_r,\beta_d}
=\lim_{\beta_r,\beta_d\to0}\mathbb E_{P_c}[w_{\beta_r,\beta_d}]
=\mathbb E_{P_c}[\mathbf1_C]=P_c(C)>0.
\end{equation}
\proofstep{Step 3: The normalized distribution converges to $\pi^\star$.} The target in Eq.~\eqref{eq:constraint} is the base-batch distribution conditioned on $C$:
\begin{equation}
\label{eq:app-target-density}
\pi^\star(x^{1:K})=\frac{P_c(x^{1:K})\mathbf1_C(x^{1:K})}{P_c(C)}.
\end{equation}
Equations~\eqref{eq:app-weight-limit} and~\eqref{eq:app-normalizer-limit} give $w_{\beta_r,\beta_d}/Z_{\beta_r,\beta_d}\to\mathbf1_C/P_c(C)$ almost everywhere under $P_c$. Division is valid because the limiting normalizer $P_c(C)$ is positive.

To conclude convergence in total variation, we need to integrate the absolute difference of these ratios. For sufficiently small temperatures, convergence of the normalizer ensures $Z_{\beta_r,\beta_d}\ge P_c(C)/2$. The triangle inequality and $0\le w_{\beta_r,\beta_d},\mathbf1_C\le1$ then give
\[
\left|\frac{w_{\beta_r,\beta_d}}{Z_{\beta_r,\beta_d}}-\frac{\mathbf1_C}{P_c(C)}\right|
\le\frac{w_{\beta_r,\beta_d}}{Z_{\beta_r,\beta_d}}+\frac{\mathbf1_C}{P_c(C)}
\le\frac{2}{P_c(C)}+\frac{1}{P_c(C)}=\frac{3}{P_c(C)}.
\]
This finite constant is integrable under the probability distribution $P_c$, so dominated convergence applies again. Total variation is half the integral of the absolute density difference. We substitute the densities from Eqs.~\eqref{eq:app-induced-distribution} and~\eqref{eq:app-target-density} into this integral. Factoring out the common density $P_c$ turns the integral into an expectation:
\begin{align}
2\left\|\widetilde\pi_{\beta_r,\beta_d}-\pi^\star\right\|_{\mathrm{TV}}
&=\int\left|\widetilde\pi_{\beta_r,\beta_d}(x^{1:K})-\pi^\star(x^{1:K})\right|dx^1\cdots dx^K\notag\\
&=\mathbb E_{P_c}\left[\left|\frac{w_{\beta_r,\beta_d}}{Z_{\beta_r,\beta_d}}-\frac{\mathbf1_C}{P_c(C)}\right|\right]\longrightarrow0.
\label{eq:app-total-variation}
\end{align}
The integrand tends to zero by the ratio limit above. Dividing by two proves convergence to $\pi^\star$ in total variation as $\beta_r,\beta_d\to0$.
\end{proof}

\paragraph{Pareto optimality of the batch targets.}
The preceding proof connects the objective to $\pi^\star$. We now show that this batch target is Pareto-optimal under the support-based diversity assumption. The argument follows the same two comparisons as the single-image theorem: a competitor cannot increase diversity while preserving the reward guarantee, and a strictly higher guarantee requires lower diversity.

\textbf{Setup.} Fix $P_c$, $D_0$, and $\delta$ as above. Write $S(x^{1:K})=\min_k R(x^k)$ for worst-candidate reward. The feasible set is then $C=\{S\ge\rho,\ D_0\ge\delta\}$, and $\pi^\star=P_c\mathbf1_C/P_c(C)$, with $P_c(C)>0$.

Assume distribution-level diversity is nonincreasing under support restriction. Assume also that it strictly decreases when a set of positive $P_c$ probability is removed. This diversity compares distributions over batches; it is distinct from the per-batch statistic $D_0$.

\proofstep{Step 1: Apply the conditioning lemma on batch space.} The conditioning calculation in Lemma~\ref{lem:iproj} uses only the positive probability of the conditioning event. We can therefore substitute $P_c$, $C$, and $\pi^\star$ for $p$, $A_\rho$, and $p_\rho$. This gives $\mathrm{KL}(\pi^\star\Vert P_c)=-\log P_c(C)<\infty$. It also shows that every finite-KL distribution assigning probability one to $C$ has support contained in that of $\pi^\star$.

\proofstep{Step 2: Preserving the reward guarantee cannot increase diversity.} Let $g$ be the highest guaranteed value of $S$ under $\pi^\star$. Here $S$ is real-valued and is at least $\rho$ with probability one under $\pi^\star$. Step 2 of Theorem~\ref{thm:frontier}'s proof therefore applies: $g$ is finite, $g\ge\rho$, and $\pi^\star(S\ge g)=1$.

Consider a competitor $\nu$ with finite KL to $P_c$ that satisfies $D_0\ge\delta$ almost surely and guarantees reward at least $g$. Because $g\ge\rho$, this competitor also satisfies the reward requirement defining $C$. The union bound makes both requirements explicit:
\[
\nu(C^c)\le\nu(S<\rho)+\nu(D_0<\delta)=0+0=0,
\qquad \nu(C)=1.
\]
The lemma therefore places $\nu$'s support within that of $\pi^\star$, so $\nu$ cannot have greater diversity.

\proofstep{Step 3: A higher reward guarantee requires lower diversity.} If $\nu$ guarantees a strictly higher floor $h>g$, then $\pi^\star(S<h)>0$. Otherwise $h$ would also be guaranteed under $\pi^\star$, contradicting the definition of $g$. In contrast, $\nu(S<h)=0$. Substituting the target density gives
\[
P_c(C\cap\{S<h\})=P_c(C)\pi^\star(S<h)>0,
\qquad \nu(C\cap\{S<h\})=0.
\]
This is a proper support restriction, so $\nu$ has strictly lower diversity. Thus $\pi^\star$ is Pareto-optimal between guaranteed worst-candidate reward and support-based diversity within the stated class of batch distributions.

At a fixed objective evaluation, $\rho=M-\Delta$. Varying $\Delta>0$ therefore traverses this target family over the feasible floors below $M$.
\endgroup

\section{Experiment Settings}
\label{app:hyperparameters}

\paragraph{Sampling and update schedules.}
The main comparisons use $499$ Pick-a-Pic prompts and $3200$ HPDv2 prompts; the $200$ DrawBench prompts are evaluated separately in Appendix~\ref{app:experiments}. We compute $95\%$ percentile bootstrap confidence intervals by resampling prompts with replacement, keeping each prompt's four candidates together and holding $\rho$ fixed. The displayed $\pm$ values are the interval half-widths. For the reward gaps at matched DreamSim in Figure~\ref{fig:corner}, we linearly interpolate FK steering's worst-candidate reward between adjacent evaluated points, without extrapolating beyond its sampled DreamSim range.

We sample FLUX at $1024\times1024$ with 28 Euler steps and guidance $3.5$, SANA at $1024\times1024$ with 20 Euler steps and guidance $4.5$, and SD1.5 at $512\times512$ with 100 DDIM~\citep{song2021denoising} steps, $\eta=1$, and guidance $7.5$. For FLUX and SANA, every method uses the marginal-preserving SDE of \citet{mark2025feynman} with noise amplitude $a=0.3$. At steps where reward is evaluated, we estimate $\hat x_0$ by averaging the velocity predictions from the current and previous steps. Because the SDE correction is applied through a first-order Euler update, we replace SANA's shipped second-order solver with Euler for each method. For methods with scheduled scoring or resampling, the common steps are $4,8,\dots,24$ on FLUX and $3,6,\dots,18$ on SANA. On SD1.5, the baselines use their published configurations, and \method follows FK steering's schedule, updating latents at steps $20,40,60,80$ and $99$. We used an NVIDIA A100 GPU for the original experiments and NVIDIA H100 GPUs for runtime and memory profiling. Unless stated otherwise, each method returns four candidates per prompt using matched starting seeds.

\paragraph{Baseline settings.}
DAS uses Srinivasan Sampling Process (SSP) resampling, while FK steering uses multinomial resampling; both resample when the effective sample size is below $0.5K$. For DAS, the annealing rate $\gamma$ is determined by the number of sampling steps, and we set $\alpha$ using its recommended range for the ratio of the reward-guidance norm to the model-drift norm. FK steering applies $e^{\lambda r}$ to the unstandardized reward, so we select $\lambda$ according to the reward scale. We use the published NegToMe settings.

\paragraph{\method settings.}
We use a fixed step-size multiplier of $1$ for the latent updates. We set $\delta$ to the median of the $\binom{K}{2}$ pairwise distances at the first update step and set $\beta_d=\tfrac12\delta$. At each update, the diversity $D$ in Eq.~\ref{eq:divpen} is the mean of the current pairwise distances weighted by $q_iq_j$. Each $d_{ij}$ is the cosine distance between pooled features from FLUX joint block 12, SANA transformer block 7, or the SD1.5 UNet bottleneck. We apply \copyrule only during the second half of denoising ($s>T/2$) for every reward and base-model setting. We recompute $\tau$ at each update step but stop gradients through $\max_k r_k$; otherwise, increasing the current maximum reward would raise $\tau$ and increase the penalties on the other candidates, producing a gradient that reduces the maximum reward. The only \method setting swept in the main experiments is the tolerance $\Delta$. For each reward and base model, the selected $\Delta$ is approximately three-quarters of the median per-prompt difference between the largest and smallest candidate rewards. We also report sweeps over $\Delta$.

\paragraph{Numerical latent updates.}
For each candidate, we normalize the reward and diversity gradients separately. We weight the normalized reward gradient by $(1-q_k)/\max_j(1-q_j)$, then clip each gradient separately to a maximum norm of $0.1$ times the mean candidate latent norm at the first update. After clipping, we multiply the reward gradient by $0.5+1.5(s-1)/(T-1)$ and subtract the sum of the two gradients from the latent. We hold the pair weights $q_iq_j$ fixed when differentiating the diversity penalty.

\paragraph{Settings across rewards and base models.}
For each reward and base model, $\rho$ is the median reward of the base model. With HPSv3, FLUX uses $\rho=9.66$, $\Delta=1.5$, DAS $\alpha=0.1$, and FK $\lambda=1.5$. With ImageReward, FLUX uses $\rho=1.23$, $\Delta=0.5$, DAS $\alpha=0.005$, and FK $\lambda=10$; SANA uses $\rho=1.19$, $\Delta=0.5$, DAS $\alpha=0.005$, and FK $\lambda=10$; and SD1.5 uses $\rho=0.18$, $\Delta=0.75$, DAS $\alpha=0.001$, and FK $\lambda=10$. DAS uses $\gamma=0.024$ on FLUX, $\gamma=0.035$ on SANA, and $\gamma=0.008$ on SD1.5, as determined by their respective sampling-step counts.

The SANA and SD1.5 qualitative examples with ImageReward use the baseline parameters listed above and $\Delta=0.5$ for \method. Each method returns four candidates from matched starting seeds.

\subsection{2D Toy Setup} \label{ssec:app-toy} 
The base model is a rectified-flow velocity network trained on a mixture of $49$ Gaussians on a $7\times7$ grid with spacing $5$, using the same interpolant as FLUX, and is frozen for all runs. The reward is the maximum of $49$ Gaussian bumps, one per mode. The $16$ bumps on every other row and column have height $1$, and the other $33$ have heights drawn in $[0.15,0.45]$. We use the floor $\rho=0.5$, so only the 16 high-reward modes contain points that satisfy the floor. The centers of the 16 high-reward bumps are displaced from the grid by Gaussian noise with standard deviation $1.5$. For the displayed \method operating point, we use $\Delta=0.325$, $\beta_r=0.2$, $\beta_d=0.5$, diversity cutoff $\delta=20$, and latent-update multiplier $0.7$, with scoring and updates at steps $3,6,\ldots,24$ of the $28$ sampling steps. These toy-specific settings were selected from a hyperparameter sweep; $\beta_r$, $\beta_d$, and $\delta$ are fixed rather than computed from each batch. Here $\beta_d$ is independent of $\delta$, unlike the $\tfrac12\delta$ used for the image models. The plotted \method metrics average all $200$ runs with seeds $42$--$241$. The candidate illustration uses seed $88$, selected to show four candidates above the floor in distinct high-reward regions; it is not an average batch. The base model, reward landscape, and baseline configurations are unchanged.

\section{Additional Experiments}
\label{app:experiments}

Unless stated otherwise, the experiments below use Pick-a-Pic prompts, with $\rho$ set to the base model's median reward and the metrics defined in Section~\ref{sec:experiments}.

\subsection{Effect of Increasing the Number of Candidates}
We test whether \method remains more diverse than reward-steering baselines when the number of candidates increases from $K{=}4$ to $K{=}16$. As Table~\ref{tab:sd15-k16} shows with the SD1.5 base model and ImageReward, \method achieves a much higher DreamSim of $0.77$ than DAS ($0.04$) or VASR ($0.47$).

\begin{table}[ht]
\centering
\caption{\textbf{At $K{=}16$, \method maintains a competitive satisfaction rate and higher diversity than the reward-steering baselines.} We use SD1.5 with ImageReward, $\Delta{=}0.75$, and $\rho{=}0.16$, the median reward of the base model's $K{=}16$ outputs.}
\label{tab:sd15-k16}
\vspace{2mm}
\begin{tabular}{l ccc c cc}
\toprule
& \multicolumn{3}{c}{Satisfaction--diversity} & \multicolumn{1}{c}{Additional metric} & \multicolumn{2}{c}{Held-out} \\
\cmidrule(lr){2-4}\cmidrule(lr){5-5}\cmidrule(lr){6-7}
Method & min $\uparrow$ & DreamSim $\uparrow$ & $\mathrm{SAT}_\rho$ $\uparrow$ & LPIPS $\uparrow$ & HPSv3 & PickScore \\
\midrule
Base           & $-1.19$ & $0.89$ & $0.50$ & $0.76$ & $5.74$ & $20.49$ \\
NegToMe        & $-1.82$ & $1.01$ & $0.30$ & $0.83$ & $2.00$ & $19.42$ \\
\midrule
FK steering    & $0.69$ & $0.45$ & $0.88$ & $0.37$ & $7.25$ & $20.91$ \\
FK + NegToMe   & $-1.03$ & $0.64$ & $0.46$ & $0.53$ & $0.43$ & $19.26$ \\
VASR       & $0.74$ & $0.47$ & $0.88$ & $0.38$ & $7.55$ & $20.96$ \\
DAS            & $1.20$ & $0.04$ & $0.86$ & $0.01$ & $-3.48$ & $18.37$ \\
\midrule
Ours ($\Delta{=}0.75$) & $0.41$ & $0.77$ & $0.83$ & $0.69$ & $6.75$ & $20.76$ \\
\bottomrule
\end{tabular}
\end{table}

\subsection{Additional Results on DrawBench} We use DrawBench~\citep{saharia2022photorealistic} to evaluate whether our improvements extend beyond the Pick-a-Pic and HPDv2 prompts used in the main text. Its prompts cover counting, spatial relations, rare combinations, and long descriptions. As Table~\ref{tab:drawbench} shows, \method attains a satisfaction rate of $0.63$ and DreamSim of $0.48$, compared with $0.60$ and $0.05$ for DAS and $0.57$ and $0.18$ for FK steering. Although NegToMe produces greater diversity, it has a substantially lower satisfaction rate of $0.34$.

\begin{table}[ht]
\centering
\caption{\textbf{\method has the highest observed satisfaction rate on DrawBench and higher observed diversity than every reward-steering method.} We use FLUX.1-dev with HPSv3 on the $200$ DrawBench prompts and set $\rho{=}11.60$.}
\label{tab:drawbench}
\vspace{2mm}
\begin{tabular}{l cc c}
\toprule
& \multicolumn{2}{c}{Satisfaction--diversity} & \multicolumn{1}{c}{Additional metric} \\
\cmidrule(lr){2-3}\cmidrule(lr){4-4}
Method & DreamSim $\uparrow$ & $\mathrm{SAT}_\rho$ $\uparrow$ & LPIPS $\uparrow$ \\
\midrule
Base          & $0.70$ & $0.50$ & $0.64$ \\
NegToMe       & $0.90$ & $0.34$ & $0.78$ \\
\midrule
FK steering   & $0.18$ & $0.57$ & $0.16$ \\
FK + NegToMe  & $0.31$ & $0.37$ & $0.27$ \\
VASR          & $0.05$ & $0.56$ & $0.03$ \\
DAS           & $0.05$ & $0.60$ & $0.03$ \\
\midrule
Ours          & $0.48$ & $0.63$ & $0.45$ \\
\bottomrule
\end{tabular}
\end{table}

\subsection{Effect of a Higher Reward Floor} We repeat the SANA comparison with $\rho$ set to the $75$th percentile rather than the median of the base model's ImageReward distribution. As Figure~\ref{fig:sana-floor-q75} shows, \method has a higher observed satisfaction rate and DreamSim than every reward-steering baseline at $\Delta\in\{0.1,0.3,0.5,0.75\}$.

\begin{figure}[ht]
\centering
\includegraphics[width=0.68\linewidth]{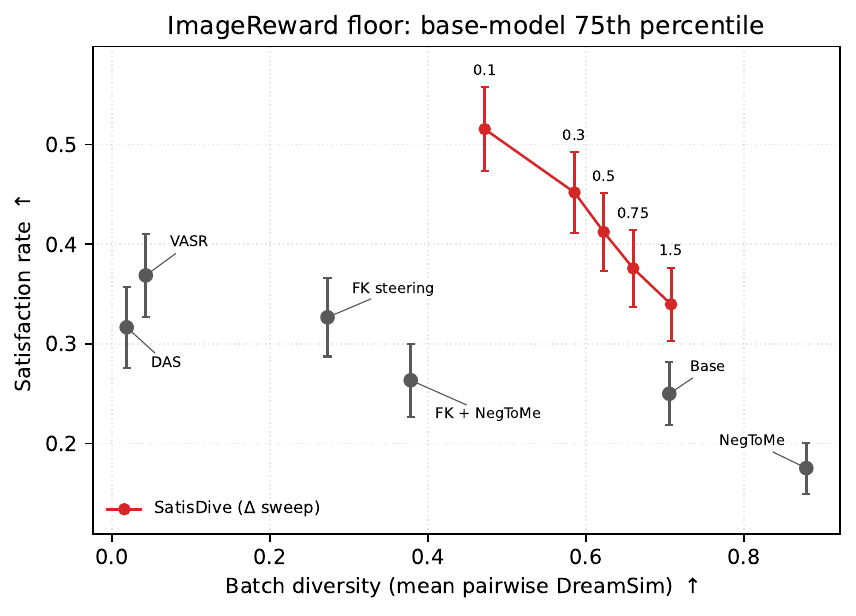}
\caption{\textbf{At $\Delta\in\{0.1,0.3,0.5,0.75\}$, \method has a higher observed satisfaction rate and DreamSim than every reward-steering baseline under the stricter floor.} We use SANA-1.6B with ImageReward and set $\rho{=}1.69$, the $75$th percentile of the base model's candidate rewards. Error bars are $95\%$ bootstrap confidence intervals over prompts. For legibility, we show five of the seven $\Delta$ settings.}
\label{fig:sana-floor-q75}
\end{figure}

\subsection{CMMD on MS-COCO}
CMMD~\citep{jayasumana2024rethinking} measures the distance between distributions of generated and real images. We compute CMMD between each method's generated images and the MS-COCO~\citep{lin2014microsoft} val2017 images to test whether \method's diversity advantage over reward-steering methods coincides with greater distance from the real-image distribution. \method obtains a CMMD of $0.81$, compared with $0.83$--$0.88$ for the baselines.

\begin{table}[H]
\centering
\caption{\textbf{\method has the lowest observed CMMD among the evaluated methods.} We report CMMD ($\times10^3$; lower is better) using $2000$ images generated from $500$ MS-COCO captions per method and all $5000$ MS-COCO val2017 images as the reference distribution. We use FLUX.1-dev with HPSv3, $\Delta{=}1.5$ for \method, $\lambda{=}1.5$ for FK steering, $\alpha{=}0.1$ and $\gamma{=}0.024$ for DAS, and the published NegToMe settings.}
\label{tab:cmmd}
\vspace{2mm}
\begin{tabular}{l c}
\toprule
Method & CMMD $\downarrow$ \\
\midrule
Base        & $0.87$ \\
NegToMe     & $0.83$ \\
FK steering & $0.88$ \\
DAS         & $0.87$ \\
\midrule
Ours        & $\mathbf{0.81}$ \\
\bottomrule
\end{tabular}
\end{table}

\subsection{Component Ablations}
Table~\ref{tab:abl-flux} removes \copyrule and the diversity penalty separately from \method to measure how each component changes satisfaction and diversity. Removing \copyrule reduces worst-candidate HPSv3 from $9.39$ to $8.58$ and satisfaction from $0.55$ to $0.50$, while removing the diversity penalty reduces DreamSim from $0.56$ to $0.54$.

\begin{table}[ht]
\centering
\caption{\textbf{Removing \copyrule lowers worst-candidate reward, while removing the diversity penalty lowers DreamSim.} We remove one component at a time from \method using FLUX.1-dev with HPSv3, $\Delta{=}1.5$, and $\rho{=}9.66$. The $95\%$ bootstrap confidence intervals computed from per-prompt differences exclude zero for the changes in worst-candidate and mean HPSv3 after removing \copyrule and for the changes in DreamSim and mean HPSv3 after removing the diversity penalty.}
\label{tab:abl-flux}
\vspace{2mm}
\small
\setlength{\tabcolsep}{5pt}
\begin{tabular}{l ccc c cc}
\toprule
& \multicolumn{3}{c}{Satisfaction--diversity} & \multicolumn{1}{c}{Additional metric} & \multicolumn{2}{c}{Held-out} \\
\cmidrule(lr){2-4}\cmidrule(lr){5-5}\cmidrule(lr){6-7}
Variant & min $\uparrow$ & DreamSim $\uparrow$ & $\mathrm{SAT}_\rho$ $\uparrow$ & mean $\uparrow$ & ImageReward & PickScore \\
\midrule
Full method           & $9.39$ & $0.56$ & $0.55$ & $10.03$ & $1.04$ & $22.32$ \\
w/o replacement       & $8.58$ & $0.68$ & $0.50$ & $9.65$ & $1.01$ & $22.27$ \\
w/o diversity penalty & $9.43$ & $0.54$ & $0.56$ & $10.07$ & $1.07$ & $22.35$ \\
\bottomrule
\end{tabular}
\end{table}

We additionally ablate \copyrule on Stable Diffusion~1.5 with ImageReward. As Table~\ref{tab:armb} shows, removing latent replacement reduces worst-candidate reward from $0.54$ to $-0.07$ and satisfaction from $0.77$ to $0.64$. DreamSim increases from $0.72$ to $0.88$ because the batch retains a low-reward candidate that the full method would replace.

\begin{table}[ht]
\centering
\caption{\textbf{\copyrule also improves worst-candidate reward and satisfaction on Stable Diffusion~1.5.} We compare the full method with a variant without latent replacement using Stable Diffusion~1.5 with ImageReward on Pick-a-Pic, $\Delta{=}0.75$, and $\rho{=}0.18$. The $\pm$ value is the half-width of a $95\%$ prompt-bootstrap confidence interval, where available.}
\label{tab:armb}
\vspace{2mm}
\begin{tabular}{l ccc cc}
\toprule
& \multicolumn{3}{c}{Satisfaction--diversity} & \multicolumn{2}{c}{Held-out} \\
\cmidrule(lr){2-4}\cmidrule(lr){5-6}
Variant & min $\uparrow$ & DreamSim & $\mathrm{SAT}_\rho$ $\uparrow$ & HPSv3 & PickScore \\
\midrule
Full method        & $0.54${\scriptsize$\,\pm0.08$} & $0.72$ & $0.77$ & $6.40$ & $20.69$ \\
w/o replacement    & $-0.07$ & $0.88$ & $0.64$ & $5.94$ & $20.57$ \\
\bottomrule
\end{tabular}
\end{table}

We also test whether the replaced candidate and the selected higher-reward candidate, whose latents are identical immediately after replacement, remain identical after subsequent denoising. The full FLUX.1-dev run in Table~\ref{tab:abl-flux} records $521$ replacement events on $314$ Pick-a-Pic prompts, all during the final two update steps. Immediately before replacement, the replaced candidate has a mean reward of $8.49$, compared with $10.87$ for the selected candidate. Across $503$ distinct pairs of replaced and selected candidates from these events, the corresponding final images have a mean DreamSim of $0.195$. Thus, identical latents do not necessarily remain identical after subsequent denoising.

\subsection{FLUX.1-dev with ImageReward}
Table~\ref{tab:flux-ir} repeats the FLUX.1-dev comparison with ImageReward instead of HPSv3 to test whether \method remains more diverse than reward-steering methods at comparable satisfaction. \method and FK steering both attain a satisfaction rate of $0.61$, while their DreamSim values are $0.57$ and $0.23$, respectively.

\begin{table}[ht]
\centering
\caption{\textbf{With ImageReward, \method matches FK steering's satisfaction while maintaining higher diversity.} We repeat the FLUX.1-dev comparison using ImageReward, $\Delta{=}0.5$, and $\rho{=}1.23$. The $\pm$ values are half-widths of $95\%$ percentile bootstrap confidence intervals over prompts.}
\vspace{2mm}
\label{tab:flux-ir}
\small
\begin{tabular}{l ccc c cc}
\toprule
& \multicolumn{3}{c}{Satisfaction--diversity} & \multicolumn{1}{c}{Additional metric} & \multicolumn{2}{c}{Held-out} \\
\cmidrule(lr){2-4}\cmidrule(lr){5-5}\cmidrule(lr){6-7}
Method & min $\uparrow$ & DreamSim $\uparrow$ & $\mathrm{SAT}_\rho$ $\uparrow$ & LPIPS $\uparrow$ & HPSv3 & PickScore \\
\midrule
Base                  & $0.68$ & $0.66${\scriptsize$\,\pm0.01$} & $0.50${\scriptsize$\,\pm0.04$} & $0.64$ & $9.65$ & $22.28$ \\
NegToMe               & $0.20$ & $0.88${\scriptsize$\,\pm0.01$} & $0.39${\scriptsize$\,\pm0.03$} & $0.79$ & $8.17$ & $21.49$ \\
\midrule
FK steering           & $1.16$ & $0.23${\scriptsize$\,\pm0.02$} & $0.61${\scriptsize$\,\pm0.04$} & $0.20$ & $9.74$ & $22.38$ \\
FK + NegToMe          & $0.97$ & $0.45${\scriptsize$\,\pm0.02$} & $0.56${\scriptsize$\,\pm0.04$} & $0.38$ & $7.45$ & $21.33$ \\
VASR                  & $1.24$ & $0.05${\scriptsize$\,\pm0.00$} & $0.62${\scriptsize$\,\pm0.04$} & $0.03$ & $9.79$ & $22.40$ \\
DAS                   & $1.18$ & $0.05${\scriptsize$\,\pm0.01$} & $0.58${\scriptsize$\,\pm0.04$} & $0.03$ & $9.75$ & $22.35$ \\
\midrule
Ours ($\Delta{=}0.5$) & $1.10$ & $0.57${\scriptsize$\,\pm0.02$} & $0.61${\scriptsize$\,\pm0.04$} & $0.55$ & $9.72$ & $22.34$ \\
\bottomrule
\end{tabular}
\end{table}

\subsection{Stable Diffusion~1.5 with ImageReward}
Table~\ref{tab:sd15-ir} tests whether the satisfaction--diversity comparison holds on SD1.5 with ImageReward. At $\Delta{=}0.5$, \method attains a satisfaction rate of $0.80$ and DreamSim of $0.65$, compared with $0.72$ and $0.49$ for FK steering and $0.77$ and $0.36$ for VASR. DAS attains a higher satisfaction rate of $0.84$ but a DreamSim of $0.04$. VASR (soft) uses reward-weighted resampling without redirecting duplicate selections to the highest-reward candidate.

\begin{table}[H]
\centering
\caption{\textbf{\method achieves higher satisfaction and diversity than FK steering on Stable Diffusion~1.5.} We use ImageReward with $\rho{=}0.18$ and report \method at three values of $\Delta$. The $\pm$ values are half-widths of $95\%$ percentile bootstrap confidence intervals over prompts.}
\label{tab:sd15-ir}
\vspace{2mm}
\small
\begin{tabular}{l ccc c cc}
\toprule
& \multicolumn{3}{c}{Satisfaction--diversity} & \multicolumn{1}{c}{Additional metric} & \multicolumn{2}{c}{Held-out} \\
\cmidrule(lr){2-4}\cmidrule(lr){5-5}\cmidrule(lr){6-7}
Method & min $\uparrow$ & DreamSim $\uparrow$ & $\mathrm{SAT}_\rho$ $\uparrow$ & LPIPS $\uparrow$ & HPSv3 & PickScore \\
\midrule
Base                   & $-0.62$ & $0.89${\scriptsize$\,\pm0.01$} & $0.50${\scriptsize$\,\pm0.03$} & $0.77$ & $5.77$ & $20.48$ \\
NegToMe                & $-1.24$ & $1.03${\scriptsize$\,\pm0.01$} & $0.32${\scriptsize$\,\pm0.03$} & $0.84$ & $2.79$ & $19.63$ \\
\midrule
FK steering            & $0.35$ & $0.49${\scriptsize$\,\pm0.03$} & $0.72${\scriptsize$\,\pm0.04$} & $0.40$ & $6.71$ & $20.78$ \\
FK + NegToMe           & $-0.63$ & $0.71${\scriptsize$\,\pm0.03$} & $0.45${\scriptsize$\,\pm0.04$} & $0.56$ & $1.63$ & $19.53$ \\
VASR                   & $0.62$ & $0.36${\scriptsize$\,\pm0.01$} & $0.77${\scriptsize$\,\pm0.04$} & $0.26$ & $7.00$ & $20.81$ \\
VASR (soft)            & $0.16$ & $0.54${\scriptsize$\,\pm0.03$} & $0.69${\scriptsize$\,\pm0.04$} & $0.44$ & $6.61$ & $20.71$ \\
DAS                    & $1.14$ & $0.04${\scriptsize$\,\pm0.01$} & $0.84${\scriptsize$\,\pm0.03$} & $0.01$ & $-4.01$ & $18.29$ \\
\midrule
Ours ($\Delta{=}0.5$)  & $0.70${\scriptsize$\,\pm0.08$} & $0.65${\scriptsize$\,\pm0.02$} & $0.80${\scriptsize$\,\pm0.03$} & $0.57$ & $6.57$ & $20.72$ \\
Ours ($\Delta{=}1.0$)  & $0.44${\scriptsize$\,\pm0.08$} & $0.76${\scriptsize$\,\pm0.02$} & $0.74${\scriptsize$\,\pm0.03$} & $0.68$ & $6.32$ & $20.66$ \\
Ours ($\Delta{=}1.25$) & $0.32${\scriptsize$\,\pm0.09$} & $0.80${\scriptsize$\,\pm0.01$} & $0.71${\scriptsize$\,\pm0.04$} & $0.71$ & $6.24$ & $20.64$ \\
\bottomrule
\end{tabular}
\end{table}

We use DAS's published SD1.5 settings. DAS attains the highest ImageReward satisfaction rate, $0.84$, but its mean held-out HPSv3 is $-4.01$, compared with $5.77$ for the base model and $6.24$--$6.57$ for \method.

\subsection{Comparison with CREPE}
\label{ssec:app-crepe}

Table~\ref{tab:crepe-sd15} compares \method with CREPE~\citep{he2026crepe} at reward scales $10$ and $100$ on the same Pick-a-Pic prompts using SD1.5 and ImageReward. At $\Delta{=}1.0$, \method achieves higher worst-candidate reward at similar DreamSim than both CREPE settings; at $\Delta{=}1.25$, it has higher observed means for both metrics.

\begin{table}[ht]
\centering
\caption{\textbf{\method improves worst-candidate reward at comparable diversity to both evaluated CREPE settings.} Both methods return $K{=}4$ images per prompt using SD1.5 with ImageReward. We use $\rho{=}0.18$ and report means on Pick-a-Pic; $\pm$ values are half-widths of $95\%$ percentile confidence intervals from $200{,}000$ prompt-bootstrap resamples, keeping each prompt's four candidates together and $\rho$ fixed.}
\label{tab:crepe-sd15}
\vspace{2mm}
\small
\begin{tabular}{l ccc c}
\toprule
& \multicolumn{3}{c}{Satisfaction--diversity} & \multicolumn{1}{c}{Additional metric} \\
\cmidrule(lr){2-4}\cmidrule(lr){5-5}
Method & min $\uparrow$ & DreamSim $\uparrow$ & $\mathrm{SAT}_\rho$ $\uparrow$ & mean $\uparrow$ \\
\midrule
CREPE (scale $10$) & $-0.33${\scriptsize$\,\pm0.09$} & $0.761${\scriptsize$\,\pm0.011$} & $0.528${\scriptsize$\,\pm0.036$} & $0.20${\scriptsize$\,\pm0.08$} \\
CREPE (scale $100$) & $0.07${\scriptsize$\,\pm0.08$} & $0.753${\scriptsize$\,\pm0.011$} & $0.660${\scriptsize$\,\pm0.034$} & $0.52${\scriptsize$\,\pm0.07$} \\
\midrule
Ours ($\Delta{=}0.5$) & $0.70${\scriptsize$\,\pm0.08$} & $0.650${\scriptsize$\,\pm0.021$} & $0.802${\scriptsize$\,\pm0.033$} & $0.88${\scriptsize$\,\pm0.08$} \\
Ours ($\Delta{=}1.0$) & $0.44${\scriptsize$\,\pm0.08$} & $0.762${\scriptsize$\,\pm0.016$} & $0.740${\scriptsize$\,\pm0.034$} & $0.71${\scriptsize$\,\pm0.08$} \\
Ours ($\Delta{=}1.25$) & $0.32${\scriptsize$\,\pm0.09$} & $0.796${\scriptsize$\,\pm0.015$} & $0.708${\scriptsize$\,\pm0.035$} & $0.63${\scriptsize$\,\pm0.08$} \\
\bottomrule
\end{tabular}
\end{table}

We adapt the released CREPE sampler to SD1.5, generating $512\times512$ images with guidance $7.5$, $64$ noise levels, and $200$ replica-exchange iterations. We evaluate effective terminal reward scales $10$ and $100$; scale $100$ corresponds to the released reward maximum, with all other settings unchanged between runs. Four outputs are retained at fixed late iterations and denoised to completion, without reward-based selection. \method uses the SD1.5 configuration in Appendix~\ref{app:hyperparameters}. We evaluate independently generated batches on the same prompts using common ImageReward and DreamSim implementations.

\subsection{Comparison with SGI}
\label{ssec:app-sgi}

We compare with SGI~\citep{parmar2026scaling}, which begins with $M{=}128$ candidates and progressively prunes them to return $K{=}4$. Its generation budget is therefore not matched to methods that maintain four candidates throughout sampling. We use HPSv3 to score individual candidates and sweep SGI's diversity weight, denoted $\lambda_{\mathrm{SGI}}$ to distinguish it from FK steering's $\lambda$. As $\lambda_{\mathrm{SGI}}$ increases, DreamSim increases from $0.647$ to $0.823$, while satisfaction rate decreases from $0.592$ to $0.507$.

\begin{table}[ht]
\centering
\caption{\textbf{Increasing SGI's diversity weight increases DreamSim but decreases satisfaction rate.} We evaluate SGI on Pick-a-Pic using FLUX.1-dev with HPSv3. SGI prunes $M{=}128$ candidates to $K{=}4$, and $\lambda_{\mathrm{SGI}}$ weights its DINO diversity score. We use $\rho{=}9.66$; the $\pm$ values are half-widths of $95\%$ bootstrap confidence intervals over prompts.}
\label{tab:sgi-flux-hpsv3}
\vspace{2mm}
\small
\setlength{\tabcolsep}{4pt}
\begin{tabular}{c ccc c}
\toprule
& \multicolumn{3}{c}{Satisfaction--diversity} & \multicolumn{1}{c}{Additional metric} \\
\cmidrule(lr){2-4}\cmidrule(lr){5-5}
$\lambda_{\mathrm{SGI}}$ & min $\uparrow$ & DreamSim $\uparrow$ & $\mathrm{SAT}_\rho$ $\uparrow$ & mean $\uparrow$ \\
\midrule
$0$   & $9.32${\scriptsize$\,\pm0.20$} & $0.647${\scriptsize$\,\pm0.012$} & $0.592${\scriptsize$\,\pm0.037$} & $10.20$ \\
$1.5$ & $9.17${\scriptsize$\,\pm0.20$} & $0.693${\scriptsize$\,\pm0.012$} & $0.577${\scriptsize$\,\pm0.037$} & $10.15$ \\
$6$   & $8.89${\scriptsize$\,\pm0.20$} & $0.744${\scriptsize$\,\pm0.013$} & $0.563${\scriptsize$\,\pm0.036$} & $10.04$ \\
$24$  & $8.41${\scriptsize$\,\pm0.22$} & $0.803${\scriptsize$\,\pm0.013$} & $0.529${\scriptsize$\,\pm0.035$} & $9.82$ \\
$96$  & $8.12${\scriptsize$\,\pm0.23$} & $0.823${\scriptsize$\,\pm0.013$} & $0.507${\scriptsize$\,\pm0.035$} & $9.65$ \\
\bottomrule
\end{tabular}
\end{table}

\subsection{Inference-Time Computation}

Figure~\ref{fig:compute-ops} reports diffusion-model passes separately from forward evaluations of auxiliary scoring models. Under the FLUX.1-dev configuration, \method uses more diffusion-model passes than Base, NegToMe, FK steering, FK~+~NegToMe, and VASR, but fewer diffusion-model passes and reward-model forward evaluations than DAS even when every partial \method backward pass is counted as a full pass. SGI uses more diffusion-model passes and repeatedly applies HPSv3 and DINO while pruning $M{=}128$ candidates to $K{=}4$. These are operation counts rather than wall-clock measurements.

For \method, the $152$ diffusion-model forward passes comprise $112$ denoising passes, $24$ additional gradient-evaluation passes, $12$ replacement-check passes, and $4$ final feature-extraction passes. Adding $24$ full and $24$ partial backward passes gives the reported upper bound of $200$. The $40$ reward evaluations comprise $24$ update evaluations, $12$ replacement checks, and $4$ final evaluations. Panel~(b) reports auxiliary-model forward evaluations.

\begin{figure}[H]
\centering
\includegraphics[width=\textwidth]{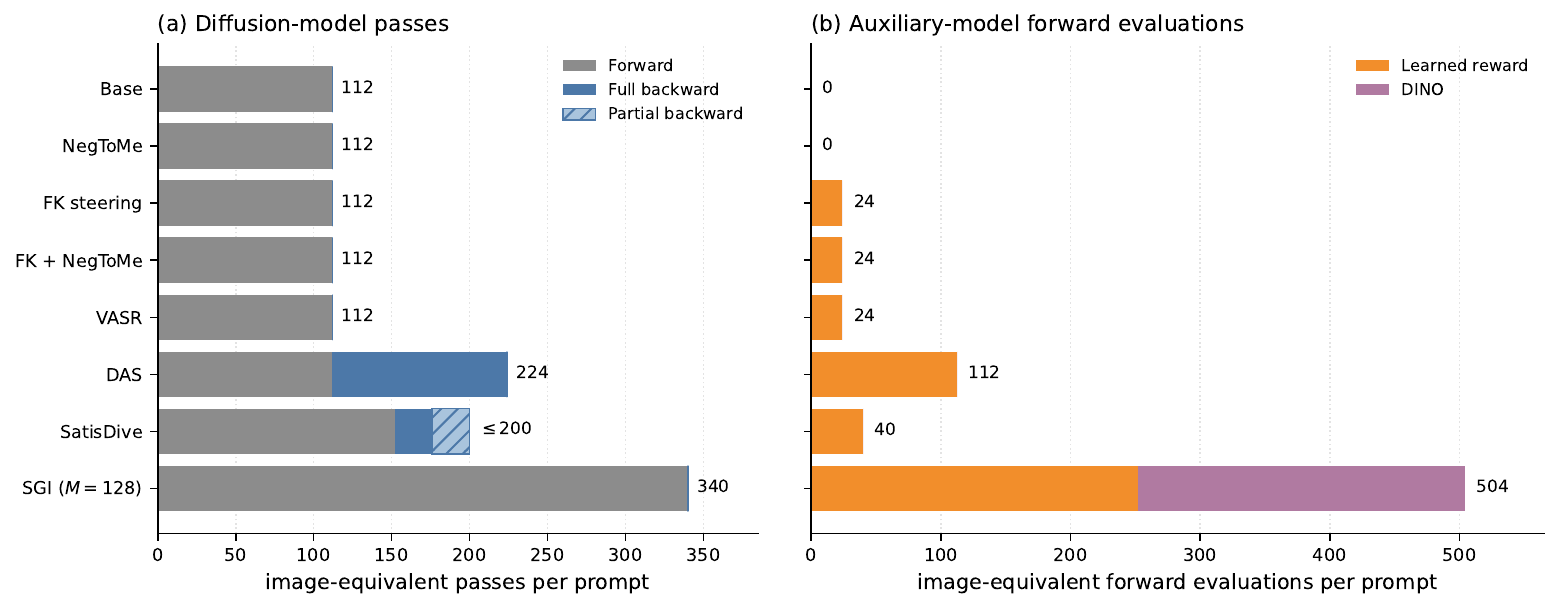}
\caption{\textbf{Inference-time operation counts for the evaluated FLUX.1-dev methods.} One image-equivalent pass processes one candidate once through the indicated model; processing a batch of $b$ candidates therefore counts as $b$ passes. We count operations per prompt for $K{=}4$ returned candidates and $28$ denoising steps. \textbf{(a)} \method uses at most $200$ diffusion-model passes, compared with $112$ for Base, NegToMe, FK steering, FK~+~NegToMe, and VASR, $224$ for DAS, and $340$ for SGI. The \method count is an upper bound because $24$ of its $48$ backward passes stop at the intermediate feature capture. DAS computes a reward gradient at every denoising step; its SSP resampling condition is checked at steps $4,8,\ldots,24$. \textbf{(b)} We count auxiliary-model forward evaluations. SGI uses $252$ HPSv3 evaluations and $252$ DINO evaluations while pruning $M{=}128$ candidates to $K{=}4$. The counts exclude VAE decoding and SGI's candidate-selection optimization.}
\label{fig:compute-ops}
\end{figure}

Table~\ref{tab:runtime-memory} separately reports measured generation time and GPU memory for the FLUX.1-dev implementations. We use the same three prompts and starting seeds for all methods, following one excluded warmup prompt.

\begin{table}[ht]
\centering
\caption{\textbf{Measured runtime and GPU memory for FLUX.1-dev.} Each prompt returns $K{=}4$ images at $1024\times1024$ resolution using $28$ denoising steps and HPSv3 for reward-guided methods. Time is the mean over three prompts, excluding model loading and image-file writes. Memory is the largest sampled resident GPU memory, summed across GPUs over generation and reward-model processes, rather than an exact peak. All runs use H100 80\,GB GPUs; two-GPU methods place the generator and reward model on separate devices. Figure~\ref{fig:corner} reports the corresponding worst-candidate reward improvements at matched diversity.}
\vspace{2mm}
\label{tab:runtime-memory}
\small
\begin{tabular}{lccc}
\toprule
Method & Time (s/prompt) & Total GPU memory (GiB) & GPUs used \\
\midrule
Base & 22.5 & 42.2 & 1 \\
NegToMe & 25.9 & 42.2 & 1 \\
FK steering & 37.7 & 58.5 & 2 \\
FK + NegToMe & 40.5 & 58.5 & 2 \\
VASR & 37.6 & 58.5 & 2 \\
DAS & 117.2 & 78.2 & 2 \\
\method & 60.8 & 90.3 & 2 \\
\bottomrule
\end{tabular}
\end{table}

\subsection{Reward-Temperature Robustness}
\label{ssec:app-seeds}

Because $\beta_r$ is computed from the standard deviation of the candidate rewards, it depends on the initial-noise seed. We repeat the Stable Diffusion~1.5 experiment with five seeds. Across these seeds, worst-candidate reward ranges from $0.54$ to $0.59$, satisfaction from $0.75$ to $0.78$, and DreamSim from $0.71$ to $0.73$. None of the pairwise differences between seeds is statistically significant for any reported metric at the $5\%$ level.

\begin{table}[ht]
\centering
\caption{\textbf{\method's satisfaction and diversity remain stable across five initial-noise seeds.} We use Stable Diffusion~1.5 with ImageReward on the same Pick-a-Pic prompts, set $\Delta{=}0.75$, and compute $\beta_r$ separately for each prompt and seed. The $\pm$ values are half-widths of $95\%$ bootstrap confidence intervals over prompts.}
\vspace{2mm}
\label{tab:sd15-seeds}
\small
\begin{tabular}{l ccc c cc}
\toprule
& \multicolumn{3}{c}{Satisfaction--diversity} & \multicolumn{1}{c}{Additional metric} & \multicolumn{2}{c}{Held-out} \\
\cmidrule(lr){2-4}\cmidrule(lr){5-5}\cmidrule(lr){6-7}
Run & min $\uparrow$ & DreamSim $\uparrow$ & $\mathrm{SAT}_\rho$ $\uparrow$ & LPIPS $\uparrow$ & HPSv3 & PickScore \\
\midrule
Seed $0$ & $0.54${\scriptsize$\,\pm0.08$} & $0.72${\scriptsize$\,\pm0.02$} & $0.77${\scriptsize$\,\pm0.03$} & $0.64$ & $6.40$ & $20.69$ \\
Seed $1$ & $0.59${\scriptsize$\,\pm0.08$} & $0.72${\scriptsize$\,\pm0.02$} & $0.76${\scriptsize$\,\pm0.03$} & $0.64$ & $6.51$ & $20.71$ \\
Seed $2$ & $0.59${\scriptsize$\,\pm0.08$} & $0.71${\scriptsize$\,\pm0.02$} & $0.76${\scriptsize$\,\pm0.03$} & $0.63$ & $6.52$ & $20.69$ \\
Seed $3$ & $0.58${\scriptsize$\,\pm0.08$} & $0.71${\scriptsize$\,\pm0.02$} & $0.78${\scriptsize$\,\pm0.03$} & $0.62$ & $6.63$ & $20.72$ \\
Seed $4$ & $0.54${\scriptsize$\,\pm0.08$} & $0.73${\scriptsize$\,\pm0.02$} & $0.75${\scriptsize$\,\pm0.03$} & $0.65$ & $6.42$ & $20.69$ \\
\bottomrule
\end{tabular}
\end{table}

\subsection{Reward Capping without the Diversity Penalty}
\label{ssec:app-dasclip}

To test whether stopping further reward optimization at $\rho$ is sufficient to preserve diversity, we modify DAS by replacing $r_k$ with $\min(r_k,\rho)$ in its guidance gradient and resampling weights. This makes rewards at or above $\rho$ equivalent in both computations without adding \method's diversity penalty. All other DAS settings remain unchanged. Reward capping increases DreamSim from $0.016$ to $0.066$, while \method reaches $0.658$ on the same prompts and reward floor. The change in DAS's satisfaction rate is not statistically significant.

\begin{table}[H]
\centering
\caption{\textbf{Reward-capped DAS remains less diverse than \method.} We use SANA-1.6B with ImageReward on Pick-a-Pic and set $\rho{=}1.19$. The reward-capped variant replaces $r_k$ with $\min(r_k,\rho)$ in the guidance gradient and resampling weights. The $\pm$ values are half-widths of $95\%$ bootstrap confidence intervals over prompts.}
\vspace{2mm}
\label{tab:dasclip}
\setlength{\tabcolsep}{5.3pt}
\begin{tabular}{l ccc c c}
\toprule
& \multicolumn{3}{c}{Satisfaction--diversity} & \multicolumn{1}{c}{Additional metric} & \multicolumn{1}{c}{Held-out} \\
\cmidrule(lr){2-4}\cmidrule(lr){5-5}\cmidrule(lr){6-6}
Method & min $\uparrow$ & DreamSim $\uparrow$ & $\mathrm{SAT}_\rho$ $\uparrow$ & mean $\uparrow$ & HPSv3 \\
\midrule
DAS                    & $1.10${\scriptsize$\,\pm0.07$} & $0.016${\scriptsize$\,\pm0.001$} & $0.599${\scriptsize$\,\pm0.043$} & $1.11$ & $7.72$ \\
DAS (reward capped)    & $1.06${\scriptsize$\,\pm0.07$} & $0.066${\scriptsize$\,\pm0.008$} & $0.577${\scriptsize$\,\pm0.042$} & $1.09$ & $7.65$ \\
\midrule
Ours ($\Delta{=}0.75$) & $1.13${\scriptsize$\,\pm0.07$} & $0.658${\scriptsize$\,\pm0.015$} & $0.674${\scriptsize$\,\pm0.039$} & $1.28$ & $7.63$ \\
\bottomrule
\end{tabular}
\end{table}

\clearpage
\section{Additional Qualitative Comparisons}
\label{app:qualitative}

Figures~\ref{fig:qual-sana-car}--\ref{fig:qual-sana-bus} extend the SANA-1.6B comparison to all seven methods. Each row contains four candidates generated from matched starting seeds and reports their minimum ImageReward and mean pairwise DreamSim; the pink outline marks the candidate with the lowest ImageReward in that row. The reward floor is $\rho_{\mathrm{IR}}{=}1.19$. Figure~\ref{fig:qual-sd15-wanderer} provides a corresponding comparison with Stable Diffusion~1.5 at its reward floor $\rho_{\mathrm{IR}}{=}0.18$.

\begin{figure}[!ht]
\centering
\includegraphics[width=0.70\textwidth]{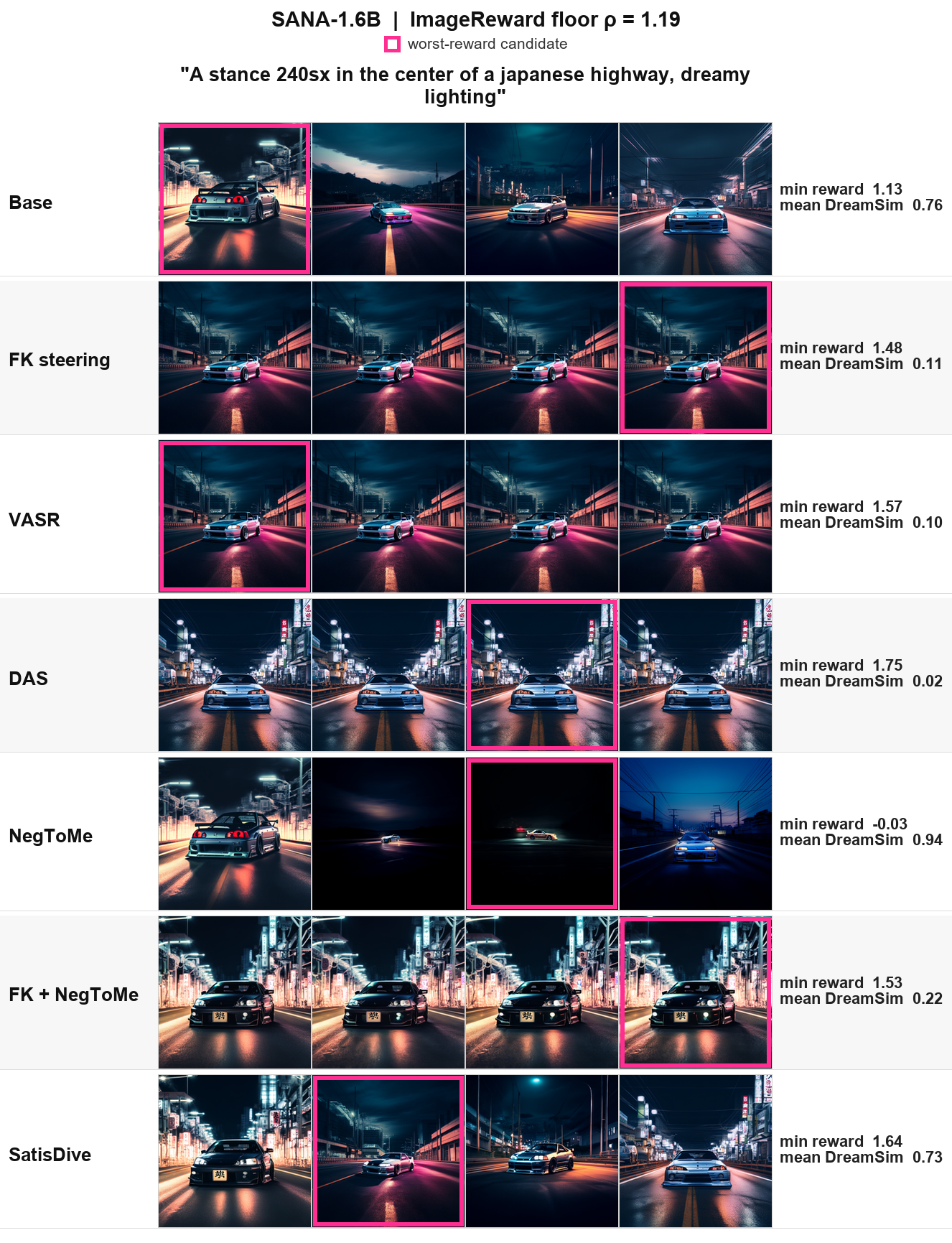}
\caption{\textbf{On the Japanese-highway car prompt, \method returns four visually distinct candidates above the reward floor.} \method attains a minimum ImageReward of $1.64$ and a DreamSim of $0.73$. FK steering, VASR, DAS, and FK~+~NegToMe also satisfy the reward floor, but their DreamSim values are $0.11$, $0.10$, $0.02$, and $0.22$. Base and NegToMe remain diverse, but their minimum rewards of $1.13$ and $-0.03$ fall below the reward floor.}
\label{fig:qual-sana-car}
\end{figure}

\begin{figure}[p]
\centering
\includegraphics[width=0.86\textwidth]{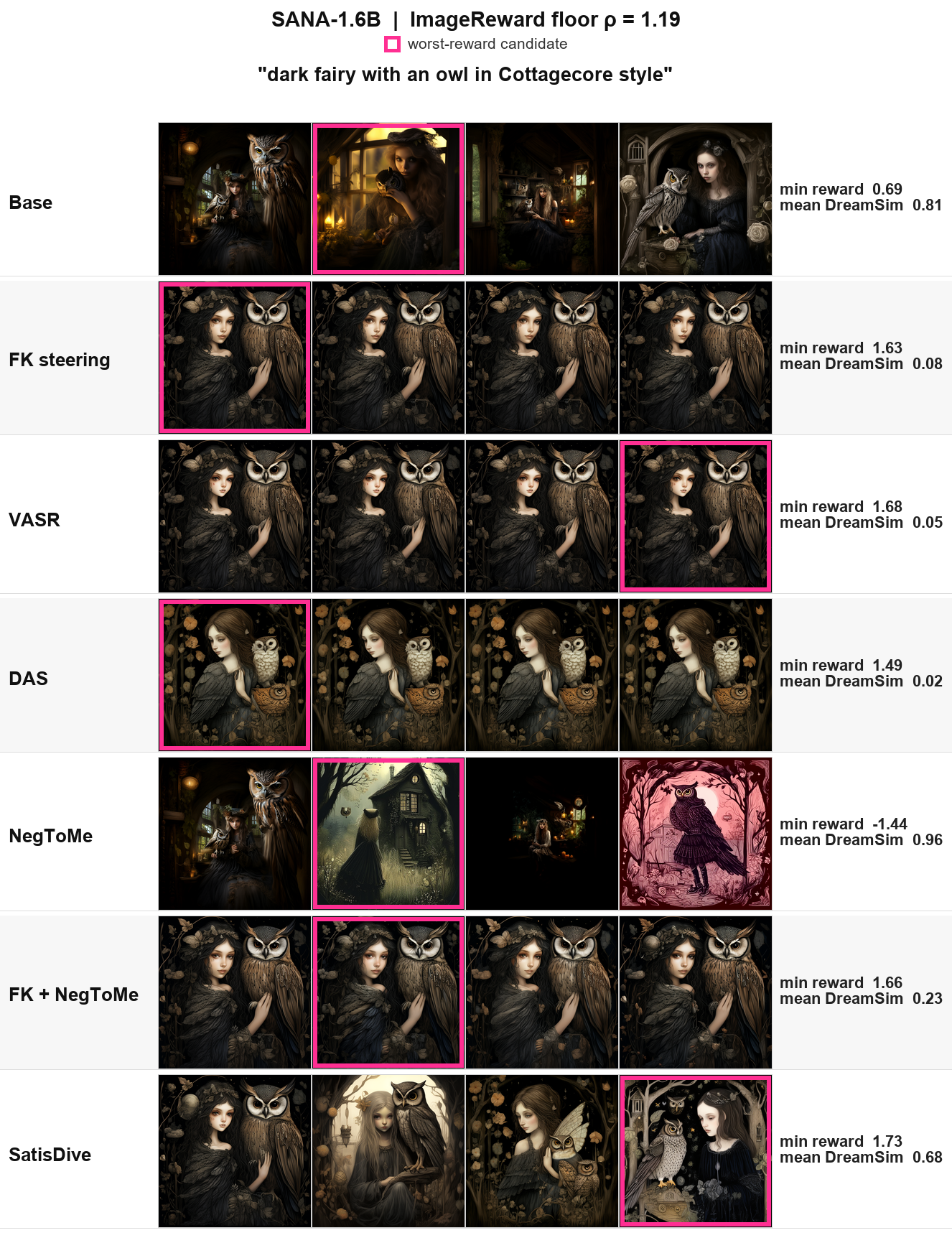}
\caption{\textbf{On the fairy-and-owl prompt, \method has the highest minimum ImageReward, $1.73$, and a DreamSim of $0.68$.} FK steering, VASR, DAS, and FK~+~NegToMe have DreamSim values between $0.02$ and $0.23$. Base and NegToMe remain diverse, but their minimum rewards of $0.69$ and $-1.44$ are below the reward floor.}
\label{fig:qual-sana-fairy}
\end{figure}

\begin{figure}[p]
\centering
\includegraphics[width=0.86\textwidth]{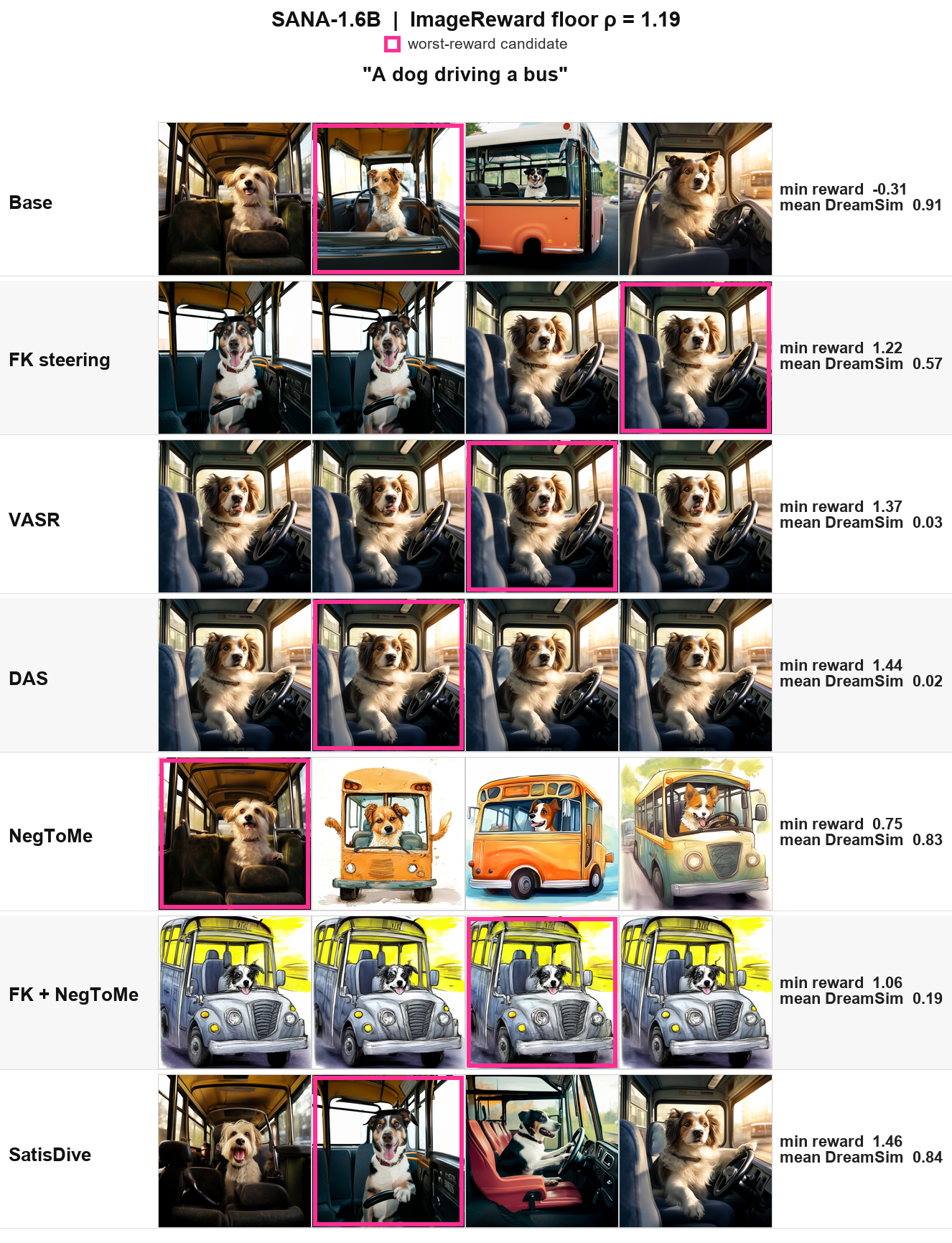}
\caption{\textbf{On the dog-driving-a-bus prompt, \method attains a minimum ImageReward of $1.46$ and a DreamSim of $0.84$.} VASR and DAS satisfy the reward floor but return visually similar candidates. Base and NegToMe have comparable diversity to \method, but each includes candidates below the reward floor.}
\label{fig:qual-sana-bus}
\end{figure}

\begin{figure}[p]
\centering
\includegraphics[width=0.86\textwidth]{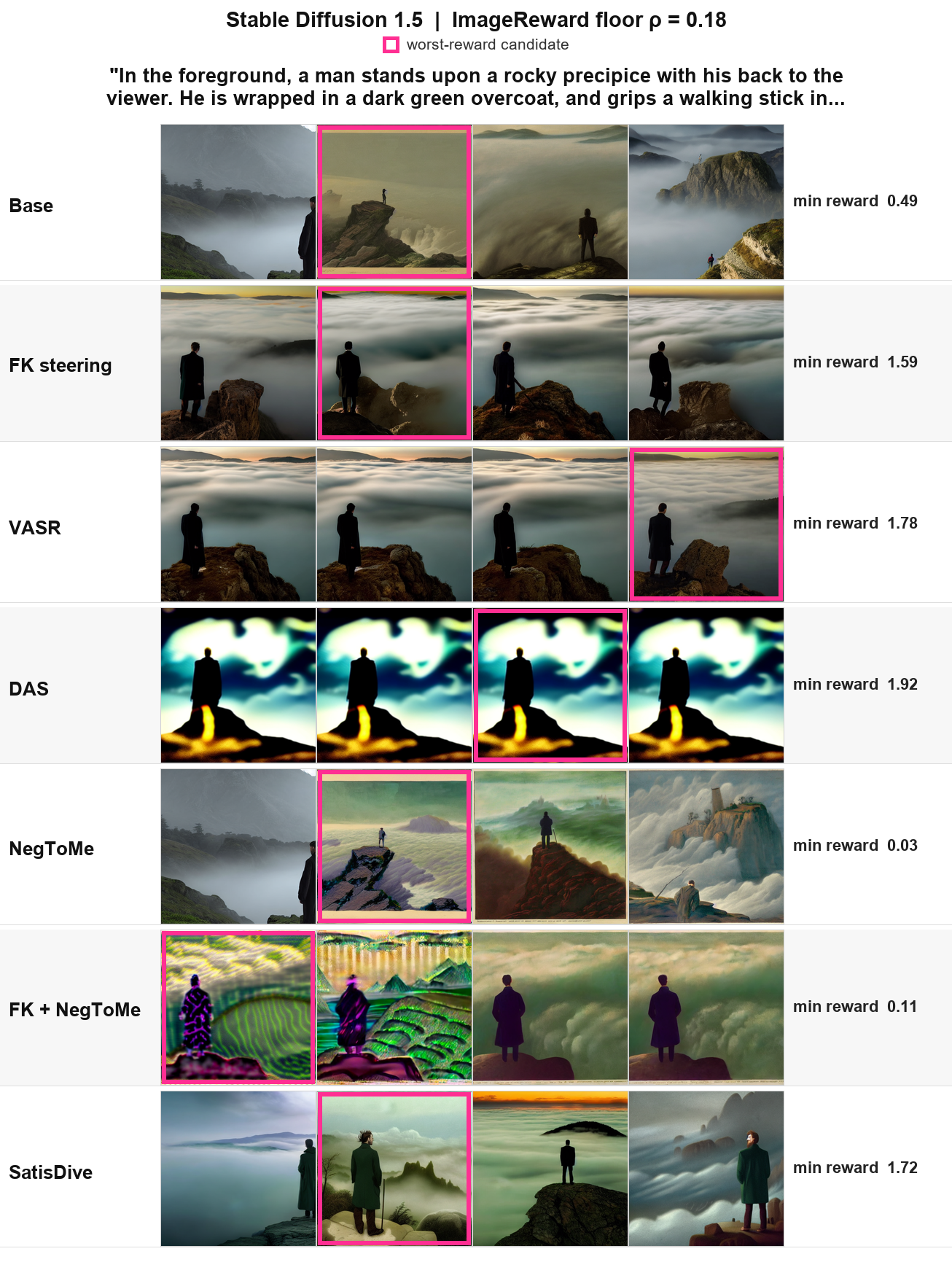}
\caption{\textbf{For Stable Diffusion~1.5, \method returns four visually distinct images with a minimum ImageReward of $1.72$.} The prompt describes a man on a rocky precipice above clouds. Base and NegToMe have minimum rewards of $0.49$ and $0.03$, while the candidates from FK steering, VASR, and DAS are visually similar within each row. All methods use matched starting seeds.}
\label{fig:qual-sd15-wanderer}
\end{figure}

\clearpage

\end{document}